%% file: fukumizu.tex
\documentclass[a4,11pt]{article}

\usepackage[utf8]{inputenc} 
\usepackage[T1]{fontenc}    
\usepackage{hyperref}       
\usepackage{url}            
\usepackage{booktabs}       
\usepackage{amsfonts}       
\usepackage{nicefrac}       
\usepackage{microtype}      

\usepackage[sectionbib]{chapterbib}

\usepackage{amsmath}
\usepackage{amssymb}
\usepackage{amsthm}
\usepackage{mathrsfs}
\usepackage[numbers,compress]{natbib}
\usepackage{bm}
\usepackage{color}
\usepackage[pdftex]{graphicx}
\usepackage{wrapfig}
\usepackage{enumitem}
\usepackage{authblk}

\usepackage{kbordermatrix}
\usepackage{comment}

\newcommand{\R}{{\mathbb{R}}}

\newcommand{\va}{{\bm{a}}}
\newcommand{\vx}{{\bm{x}}}
\newcommand{\vy}{{\bm{y}}}
\newcommand{\vz}{{\bm{z}}}
\newcommand{\vw}{{\bm{w}}}
\newcommand{\vv}{{\bm{v}}}
\newcommand{\vu}{{\bm{u}}}
\newcommand{\vf}{{\bm{f}}}
\newcommand{\vb}{{\bm{b}}}
\newcommand{\vs}{{\bm{s}}}
\newcommand{\vr}{{\bm{r}}}

\newcommand{\vbeta}{{\bm{\beta}}}
\newcommand{\vgamma}{{\bm{\gamma}}}
\newcommand{\vzeta}{{\bm{\zeta}}}
\newcommand{\veta}{{\bm{\eta}}}
\newcommand{\vxi}{\bm{\xi}}
\newcommand{\vomega}{\bm{\omega}}
\newcommand{\vpsi}{\bm{\psi}}
\newcommand{\vdelta}{{\bm{\delta}}}
\newcommand{\vlambda}{{\bm{\lambda}}}
\newcommand{\vtheta}{{\bm{\theta}}}
\newcommand{\thH}{\bm{\theta}^{(H)}}
\newcommand{\thHz}{\bm{\theta}^{(H_0)}}
\newcommand{\thHo}{\bm{\theta}^{(1)}}
\newcommand{\NN}{{\mathcal{N}}}

\newcommand{\la}{\langle}
\newcommand{\ra}{\rangle}

\newcommand{\tr}{\mathrm{Tr}}

\newcommand{\eq}[1]{Eq.~(\ref{#1})}

\newtheorem{thm}{Theorem}

\newtheorem{lma}[thm]{Lemma}
\newtheorem{prop}[thm]{Proposition}

\title{Semi-flat minima and saddle points by embedding neural networks to overparameterization}

\author[1,2]{Kenji Fukumizu}
\author[2]{Shoichiro Yamaguchi}
\author[1]{Yoh-ichi Mototake}
\author[1,3]{Mirai Tanaka}
\affil[1]{The Institute of Statistical Mathematics. 
  Tachikawa, Tokyo 190-8562, Japan. \texttt{\{fukumizu, mototake, mirai\}@ism.ac.jp} }
\affil[2]{Preferred Networks, Inc.  Chiyoda-ku, Tokyo 100-0004, Japan.   \texttt{guguchi@preferred.jp}}
\affil[3]{RIKEN.   Chuo-ku, Tokyo 103-0027, Japan}

\begin{document}

\maketitle

\begin{abstract}
  We theoretically study the landscape of the training error for neural networks in overparameterized cases.  We consider three basic methods for embedding a network into a wider one with more hidden units, and discuss whether a minimum point of the narrower network gives a minimum or saddle point of the wider one. Our results show that the networks with smooth and ReLU activation have different partially flat landscapes around the embedded point.  We also relate these results to a difference of their generalization abilities in overparameterized realization. 
\end{abstract}

\section{Introduction}

Deep neural networks (DNNs) have been applied to many problems with remarkable successes.  On the theoretical understanding of DNNs, however, many problems are still unsolved.   Among others, local minima are important issues on learning of DNNs; existence of many local minima is naturally expected by its strong nonlinearity, while people also observe that, with a large network and the stochastic gradient descent, training of DNNs may avoid this issue \cite{Keskar_ICLR2017_OnLT,pmlr-v80-kleinberg18a}.  For a better understanding of learning, it is essential to clarify the landscape of the training error.

This paper focuses on the error landscape in {\em overparameterized} situations, where the number of units is surplus to realize a function. This naturally occurs when a large network architecture is employed, and has been recently discussed in connection to optimization and generalization of neural networks (\cite{Arora_etal_icml2018,Allen-Zhu_etal_arxiv2018_overparam} to list a few).  
To formulate overparameterization rigorously, this paper introduces three basic methods, unit replication, inactive units, and inactive propagation, for embedding a network to a network of more units in some layer.  
We investigate especially the landscape of the training error around the embedded point, when we embed a minimizer of the error for a smaller model.  

A relevant topic to this paper is {\em flat minima}  \cite{Hochreiter_1995_flatminima,Hochreiter_Schnidhuber_1997_flatminima}, which have been attracting much attention in literature.  Such flatness of minima is often observed empirically, and is connected to generalization performance \cite{Chaudhari_ICLR2017_EntropySGDBG,Keskar_ICLR2017_OnLT}. There are also some works on how to define flatness appropriately and its relations to generalization \cite{Rangamani_etal_arXiv2019,Tsuzuku_etal_arXiv2019}.  Different from these works, this paper shows some embeddings cause {\em semi-flat} minima, at which a lower dimensional affine subset in the parameter space gives a constant value of error.  
We will also discuss difference between smooth activation and Rectified Linear Unit (ReLU);  at a semi-flat minimum obtained by embedding a network of zero training error, the ReLU networks have more flat directions.  Using PAC-Bayes arguments \cite{McAllester1999}, we relate this to the difference of generalization bounds between ReLU and smooth networks in overparameterized situations.

This paper extends \cite{localmin}, in which the three embedding methods are discussed and some conditions on minimum points are shown.  However, the paper is limited to three-layer networks of smooth activation with one-dimensional output, and the addition of only one hidden unit is discussed.  The current paper covers a much more general class of networks including ReLU activation and arbitrary number of layers, and discusses the difference based on the activation functions as well as a link to generalization.  

The main contributions of this paper are summarized as follows.
\begin{itemize}[topsep=0pt, partopsep=0pt, itemsep=0pt, parsep=0pt, leftmargin=11pt]
\item Three methods of embedding are introduced for the general $J$-layer networks as basic construction of overparameterized realization of a function. 
\item For smooth activation, the unit replication method embeds a minimum to a saddle point  under some assumptions. 
\item It is shown theoretically that, for ReLU activation, a minimum is always embedded as a minimum by the method of inactive units. The surplus parameters correspond to a flat subset of the training error.  The unit replication gives only a saddle point under mild conditions.
\item When a network attains zero training error, the embedding by inactive units gives semi-flat minima in both activation models. It is shown that ReLU networks give flatter minima in the overparameterized realization, which suggests better generalization through the PAC-Bayes bounds.  
\end{itemize}


All the proofs of the technical results are given in Supplements.

\section{Neural network and its embedding to a wider model}
\label{sec:embed_def}

We discuss $J$ layer, fully connected neural networks that have an activation function $\varphi(\vz;\vw)$, where $\vz$ is the input to a unit and $\vw$ is a parameter vector.  The output of the $i$-th unit $\mathcal{U}_i^q$ in the $q$-th layer is recursively defined by $z^q_i = \varphi(\vz^{q-1}; \vw^{q}_i)$, 
where 
$\vw_i^q$ is the weight between $\mathcal{U}_i^q$ and the $(q-1)$-th layer.  
The activation function $\varphi(\vz;\vw)$ is any nonlinear function, which often takes the form $\varphi(\vw_{wgt}^T\vz - w_{bias})$ with $\vw=(\vw_{wgt},w_{bias})$; typical examples are the sigmoidal function $\varphi(\vz;\vw) = \tanh(\vw_{wgt}^T\vz -w_{bias})$ and ReLU $\varphi(\vz;\vw)=\max\{\vw_{wgt}^T\vz-w_{bias}, 0\}$.  This paper assumes that there is $\vw^{(0)}$ such that $\varphi(\vx;\vw^{(0)})=0$ for any $\vx$.  
Focusing the $q$-th layer, with size of the other layers fixed, the set of networks having $H$ units in the $q$-th layer is denoted by $\NN_H$.  
With a parameter $\thH=(W_0,\vw_1,\ldots,\vw_H,\vv_1,\ldots,\vv_H,V_0)$,  the function $\vf^{(H)}_{\thH}$ of a network in $\NN_H$ is defined by
\begin{equation}\label{eq:DNN}
    \vf^{(H)}_{\thH}(\vx):=\vf^{(H)}(\vx;\thH)= \vpsi\bigl( \textstyle{\sum_{j=1}^H} \vv_j \varphi(\vx; \vw_j, W_0); V_0\bigr),
\end{equation}
where $\varphi(\vx; \vw_j,W_0)$ is the output of $\mathcal{U}_i^q$ with a summarized parameter $W_0$ in the previous layers, and $\vpsi(\vz^{q+1};V_0)$ is all the parts after $\vz^{q+1}$ with parameter $V_0$. Note that $\vv_j$ is a connection weight from the unit $\mathcal{U}^q_j$ to the units in the $(q+1)$-th layer (we omit the bias term for simplicity).  The number of units in the $(q-1)$-th and $(q+1)$-th layers are denoted by $D$ and $M$, respectively.   

Embedding of a network refers to a map associating a {\em narrower} network in $\NN_{H_0}$ ($H_0<H$) with a network of a specific parameter in a {\em wider} model $\NN_H$ to realize the same function, keeping other layers unchanged. 
For clarity, we use symbols $(\vzeta_i,\vu_i)$ instead of $(\vv_j, \vw_j)$ for the parameter $\thHz$ of $\NN_{H_0}$; 
\begin{equation}
    \vf^{(H_0)}_{\thHz}(\vx):=\vf^{(H_0)}(\vx;\thHz) =  \vpsi\bigl( \textstyle{\sum_{i=1}^{H_0}} \vzeta_i \varphi(\vx; \vu_i, W_0); V_0\bigr).
\end{equation}

We consider minima and stationary points of the {\em empirical risk} (or {\em training error})
\begin{equation}\label{eq:loss}
     L_H(\thH):= {\textstyle \sum_{\nu=1}^n} \ell(\vy_\nu, \vf^{(H)}(\vx_\nu; \thH) ),
\end{equation}
where $\ell(\vy,\vf)$ is a loss function to measure the discrepancy between a teacher $\vy$ and network output $\vf$, and $(\vx_1,\vy_1),\ldots,(\vx_n,\vy_n)$ are given training data.  Typical examples of $\ell(\vy,\vf)$ include the square error $ \|\vy - \vf\|^2/2$ and logistic loss $-y \log f - (1-y)\log (1-f)$ for $y\in \{0,1\}$ and $f\in(0,1)$. 
 In the sequel, we assume the second order differentiability of $\ell(\vy,\vf)$ with respect to $\vf$ for each $\vy$.  

\subsection{Three embedding methods of a network}
\label{sec:embed_general}

\begin{figure}[t]
    \centering
    \includegraphics[width=\textwidth]{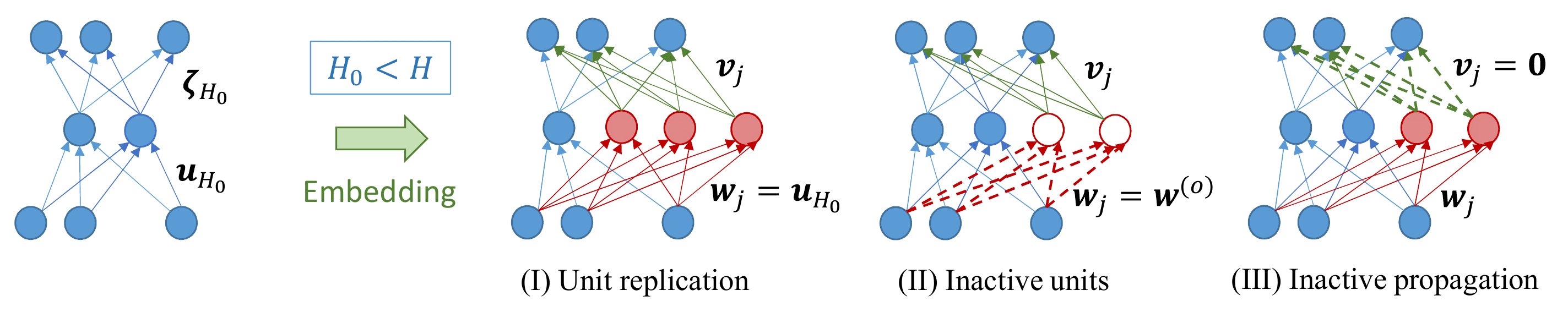}
    \vskip -3mm
    \caption{Embedding of a narrower network to a wider one.}
    \label{fig:embedding}
\end{figure}
\begin{table}[t]
    \centering
    \begin{tabular}{l|l|l}
    \hline
    \multicolumn{1}{c|}{Unit replication  $\Pi_{repl}(\thHz)$} & 
    \multicolumn{1}{c|}{Inactive units $\Pi_{iu}(\thHz)$} & 
    \multicolumn{1}{c}{Inactive propagation $\Pi_{ip}(\thHz)$}
   \\ \hline 
      $\vw_i = \vu_i$ ($1\leq i\leq H_0-1$)  &  
        $\vw_i = \vu_i$  ($1\leq i\leq H_0$)  &  
        $\vw_i = \vu_i$ ($1\leq i\leq H_0$) \\
      $\vv_i =  \vzeta_i$ ($1\leq i \leq H_0-1$)  & 
        $\vv_i =  \vzeta_i$ ($1\leq i \leq H_0$)  & 
        $\vv_i =  \vzeta_i$  ($1\leq i \leq H_0$)   \\
      $\vw_{H_0} = \cdots =\vw_H = \vu_{H_0}$ & 
         $\vw_{H_0+1}=\cdots=\vw_H=\vw^{(o)}$ & 
        $\vw_{H_0+1}, \ldots,\vw_H$: arbitrary  \\
      $\vv_{H_0}+\cdots + \vv_{H} = \vzeta_{H_0}$ & 
        $\vv_{H_0+1},\ldots,\vv_{H}$: arbitrary  & 
        $\vv_{H_0+1}=\cdots = \vv_{H} = 0$  \\
        \hline
    \end{tabular}
    \caption{Three methods of embedding}
    \vskip -5mm
    \label{tbl:embedding}
\end{table}

To fomulate overparameterization, 
we introduce three basic methods for embedding $\vf^{(H_0)}_{\thHz}$ into $\NN_H$ so that it realizes exactly the same function as $\vf^{(H_0)}_{\thHz}$.  See Table \ref{tbl:embedding} and Figure \ref{fig:embedding} for the definitions. 

{\bf (I) Unit replication: }
We fix a unit, say the $H_0$-th unit $\mathcal{U}^q_{H_0}$, in $\NN_{H_0}$, and replicate it. 
Simply, $\thH$ has $H-H_0+1$ copies of $\vu_{H_0}$, and divides the weight $\vzeta_{H_0}$ by $\vv_{H_0},\ldots,\vv_H$, keeping the other parts unchanged.  A choice of $\vu_i$ ($1\leq i\leq H_0$) to replicate is arbitrary, and a different choice defines a different network.  We use $\vu_{H_0}$ for simplicity.  
The parameters $\vv_{H_0},\ldots,\vv_H$ consist of an $(H-H_0)\times M$ dimensional affine subspace, denoted by $\Pi_{repl}(\thHz)$, in the parameters for $\NN_{H}$. 

{\bf (II) Inactive units: }
This embedding uses the special weight $\vw^{(0)}$ to make the surplus units inactive. The set of parameters is denoted by  $\Pi_{iu}(\thHz)$, which is of $(H-H_0)\times M$ dimension. 

{\bf (III) Inactive propagation: }   
This embedding cuts off the weights to the $(q+1)$-th layer for the surplus part.  The weights $\vw_j$ of  the surplus units are arbitrary. The set of parameters is denoted by  $\Pi_{ip}(\thHz)$, which is of $(H-H_0)\times D$ dimension. 

All the above embeddings give the same function as the narrower network. 
\begin{prop}\label{prop:equal_fun}
For any $\thH\in \Pi_{repl}(\thHz)\cup \Pi_{iu}(\thHz)\cup\Pi_{ip}(\thHz)$, we have 
 $   \vf^{(H)}_{\thH} = \vf^{(H_0)}_{\thHz}$.
\end{prop}

It is important to note that a network is not uniquely embedded in a wider model, in contrast to fixed bases models such as the polynomial model.  This unidentifiability has been clarified for three-layer networks 
\cite{KurkovaKainen1994,Sussmann_1992}; in fact, for three layer networks of $\tanh$ activation, \cite{Sussmann_1992} shows that the three methods essentially cover all possible embedding. 
For three-layer networks of 1-dimensional output and smooth activation, \cite{localmin} shows that this unidentifiable embedding causes minima or saddle points.  The current paper extends this result to general networks with ReLU as well as smooth activation.  




\section{Embedding of smooth networks}
\label{sec:embed_smooth}

This section assumes the second order differentiability of $\varphi(\vx;\vw)$ on $\vw$.  The case of ReLU will be discussed in Section \ref{sec:min_relu}. 
Let $\thHz_*$ be a stationary point of $L_{H_0}$, i.e., 
$\frac{\partial L_{H_0}(\thHz_*)}{\partial \thHz} = {\bf 0}$.  
We are interested in whether the embedding in Section \ref{sec:embed_def} also gives a stationary point of $L_H$. More importantly, we wish to know if a minimum of $L_{H_0}$ is embedded to a minimum of $L_H$.  A network can be embedded by any combination of the three methods, but we consider their effects separately for simplicity.

\subsection{Stationary properties of embedding}

To discuss the stationarity for the case (I) unit replication, we need to restrict $\Pi_{repl}(\thHz)$ to a subset.  For $\thHz$, define $\thH_\vlambda$ for every $\vlambda=(\lambda_{H_0},\ldots,\lambda_H)\in \R^{H-H_0+1}$ with $\sum_{j=H_0}^H \lambda_j = 1$ 
by 
\begin{align}
    & \vw_i = \vu_i, \quad \vv_i =  \vzeta_i\quad (1\leq i\leq H_0-1),\nonumber \\
    & \vw_{H_0} = \cdots =\vw_H = \vu_{H_0},\qquad 
     \vv_j = \lambda_j \vzeta_{H_0} \quad (H_0\leq j \leq H).  \label{eq:embed}
\end{align}
Obviously, $\thH_\vlambda\in\Pi_{repl}(\thHz)$ so that $\vf^{(H)}_{\thH_\vlambda} = \vf^{(H_0)}_{\thHz}$.  
The next theorem shows that a stationary point of $\NN_{H_0}$ is embedded to an $(H-H_0)$-dimensional stationary subset of $\NN_{H}$.
\begin{thm}\label{thm:stationary}
Let $\thHz_*$ be a stationary point of $L_{H_0}$. Then, for any $\vlambda=(\lambda_{H_0},\ldots,\lambda_H)$ with $\sum_{j=H_0}^H\lambda_j=1$, the point $\thH_\vlambda$ defined by \eq{eq:embed} is a stationary point of $L_H$. 
\end{thm}
The basic idea for the proof is to separate the subset of parameters $(\vv_{H_0},\vw_{H_0},\ldots,\vv_H,\vw_H)$ into a copy of  $(\vzeta_{H_0},\vu_{H_0})$ and the remaining ones, the latter of which do not contribute to change the function $\vf^{(H)}_{\thH}$ at $\thH_\vlambda$.  We will see this reparameterization in Section \ref{sec:embed_min} in detail.

It is easy to see that the embedding by inactive units or propagations does not generally embed a stationary point to a stationary one.  The details will be given in Section \ref{sec:inactive_smooth}, Supplements. 


\subsection{Embedding of a minimum point in the case of smooth networks}
\label{sec:embed_min}

We next consider the embedding $\thH_\vlambda$ of a mininum point $\thHz_*$ of $L_{H_0}$.  In the sequel, for notational simplicity, we discuss three-layer models ($J=3$) and linear output units.  For general $J$, the derivatives and Hessian of $L_H$ for the other parameters are exactly the same as those of $L_{H_0}$ for the corresponding parameters, and we omit the full description.  
The two models are simply given by
\begin{equation}\label{eq:3NN}
\NN_{H}:    \vf^{(H)}(\vx; \thH) = {\textstyle \sum_{j=1}^H} \vv_j \varphi(\vx; \vw_j) \quad \text{and}\quad  
    \NN_{H_0}: \vf^{(H_0)}(\vx; \thHz) = {\textstyle \sum_{i=1}^{H_0}} \vzeta_i \varphi(\vx; \vu_i).
\end{equation}

To simplify the Hessian for unit replication, we introduce a new parameterization of $\NN_{H}$.  
Let $\vlambda\in \R^{H-H_0+1}$ be fixed such that $\lambda_{H_0}+\cdots+\lambda_H=1$ and $\lambda_j\neq 0$.  For such $\vlambda$, take an $(H-H_0)\times (H-H_0+1)$ matrix $A=(\alpha_{cj})$  ($H_0+1\leq c\leq H, H_0\leq j\leq H$) that satisfies the two conditions:

\begin{enumerate}[topsep=0pt, partopsep=0pt, itemsep=0pt, parsep=0pt]
\item[(A1)]
    $\bigl( \begin{smallmatrix} {\bf 1}_{H-H_0+1}^T \\ A \end{smallmatrix}\bigr)$
is invertible, where ${\bf 1}_d=(1,\ldots,1)^T\in\R^d$,
\item[(A2)]
$\sum_{j=H_0}^H \alpha_{cj}\lambda_j = 0$ 
for any $H_0+1\leq c\leq H$. 
\end{enumerate} 

To find such $A$, take $A=(\bm{a}_{H_0+1},\ldots,\bm{a}_{H})^T$ so that $\bm{a}_c^T\vlambda=0$.  Then, if $\sum_{c=H_0+1}^{H} s_c \bm{a}_c={\bf 1}_{H-H_0+1}$ for some scalars $s_c$, taking the inner product with $\vlambda$ causes a contradiction. 

Given such $\vlambda$ and $A=(\alpha_{cj})$, define a bijective linear transform from $(\vv_{H_0},\ldots,\vv_H;\vw_{H_0},\ldots,\vw_H)$ to  $(\va,\vxi_{H_0+1},\ldots,\vxi_{H};\vb, \veta_{H_0+1},\ldots,\veta_{H})$ by 
\begin{equation}\label{eq:new_param}
        \vw_j = \vb + {\textstyle \sum_{c=H_0+1}^H} \alpha_{cj}\veta_c \quad \text{and} \quad 
\vv_j  = \lambda_j \va + {\textstyle \sum_{c=H_0+1}^H} \lambda_j \alpha_{cj}\vxi_c  \qquad (H_0\leq j\leq H).
\end{equation}


The parameter $\vb$ serves as the direction that makes all the hidden units behave equally, and $(\veta_j)$ define the remaining $H-1$ directions that differentiate them.   The parameter $\vb$ thus essentially plays the role of $\vu_{H_0}$ for $\NN_{H_0}$. Also, $\va$ works as $\vzeta_{H_0}$ when all $\vw_j$ are equal. 
The next lemma confirms this role of $(\va,\vb)$ and shows that the directions $\veta_c$ and $\vxi_c$ do not change the function $\vf^{(H)}$ at $\thHz_\vlambda$.
\begin{lma}\label{lma:1st_der}
Let $\thHz$ be any parameter of $\NN_{H_0}$, and $\thH_\vlambda$ be its embedding defined by \eq{eq:embed}.  Then, 
\begin{align}
    & \textstyle{ \frac{\partial \vf^{(H)}(\vx; \thH) }{\partial \vb}\Bigl|_{\thH=\thH_\vlambda}  
     =  \frac{\partial \vf^{(H_0)}(\vx; \thHz) }{\partial \vu_{H_0}},  }
    &  \quad    \textstyle{\frac{\partial \vf^{(H)}(\vx; \thH) }{\partial \veta_c}\Bigl|_{\thH=\thH_\vlambda}  
     = {\bf 0}, }
\nonumber \\
    & \textstyle{\frac{\partial \vf^{(H)}(\vx; \thH) }{\partial \va}\Bigl|_{\thH=\thH_\vlambda}  
     =  \frac{\partial \vf^{(H_0)}(\vx; \thHz) }{\partial \vzeta_{H_0}},  }
    &   \quad \textstyle{  \frac{\partial \vf^{(H)}(\vx; \thH) }{\partial \vxi_c}\Bigl|_{\thH=\thH_\vlambda}  
     = {\bf 0}.}
\end{align}
\end{lma}
From Lemma \ref{lma:1st_der}, the Hessian takes a simple form:
\begin{lma}\label{lma:Hessian}
Let $\vlambda$ and $A$ be as above.  Suppose $\thHz_*$ is a stationary point of $\NN_{H_0}$ and $\thH_\vlambda$ is its embedding defined by \eq{eq:embed}.  Then, the Hessian matrix of $L_H$ with respect to $\vomega=(\va,\vb,\vxi_{H_0+1},\ldots,\vxi_H,\veta_{H_0+1},\ldots,\veta_H)$ at $\thH=\thH_\vlambda$ is given by 
\begin{equation}\label{eq:Hessian}
\frac{\partial^2 L_H(\thH_\vlambda)}{\partial\vomega\partial\vomega}=
\kbordermatrix{
 & \va &  \vb  & \vxi_d & \veta_d \\ 
\va & \frac{\partial^2 L_{H_0}(\thHz_*)}{\partial\vzeta_{H_0}\partial\vzeta_{H_0}} &  \frac{\partial^2 L_{H_0}(\thHz_*)}{\partial\vzeta_{H_0}\partial\vu_{H_0}} &  O  & O \\ 
\vb & {\frac{\partial^2 L_{H_0}(\thHz_*)}{\partial\vu_{H_0}\partial\vzeta_{H_0}}}  &  {\frac{\partial^2 L_{H_0}(\thHz_*)}{\partial\vu_{H_0}\partial\vu_{H_0}}}  & O  & O \\ 
 \vxi_c & O &  O  & O &   \tilde{F} \\ 
 \veta_c & O  & O  & \tilde{F}^T & \tilde{G}
}.
\end{equation}
The lower-right block $\tilde{G}:=(\frac{ \partial^2  L_H(\thH_\vlambda)}{\partial{\veta_c}\partial{\veta_d}})_{cd}$, which is a symmetric matrix of $ (H-H_0)\times D$ dimension, is given by
$\left( A \Lambda A^T \right) \otimes G$ 
with $\Lambda=\text{Diag}(\lambda_{H_0},\ldots,\lambda_H)$ and  
$
G:= {\textstyle \sum_{\nu=1}^n \frac{\partial \ell(\vy_\nu,\vf^{(H_0)}(\vx_\nu;\thHz_*))}{\partial \vz} \vzeta_{H_0*} \frac{\partial^2 \varphi(\vx_\nu; \vu_{H_0*})}{\partial \vu_{H_0}\partial\vu_{H_0}}}
$; 
and  $\tilde{F}:=(\frac{\partial^2 L_H(\thH_\vlambda)}{\partial{\vxi_c}\partial{\veta_d}})_{cd}$, which is of size $ (H-H_0)\times M$ dimension, is given by $
\left( A \Lambda A^T \right) \otimes F$
with 
$
F:=  {\textstyle\sum_{\nu=1}^n \frac{\partial \ell(\vy_\nu,\vf^{(H_0)}(\vx_\nu;\thHz_*))}{\partial \vz}  \frac{\partial \varphi(\vx_\nu; \vu_{H_0*})}{\partial \vu_{H_0}}}.
$
\end{lma}

Lemma \ref{lma:Hessian} shows that, with the reparametrization, the Hessian at the embedded stationary point $\thH_\vlambda$ contains the Hessian of $L_{H_0}$ with $\va, \vb$, and that the cross blocks between $(\va,\vb)$ and $(\vxi_c,\veta_d)$ are zero.  Note that the $\vxi$-$\vxi$ block is zero, which is important when we prove Theorem \ref{thm:flat}.

\begin{thm}\label{thm:flat}
Consider a three layer neural network given by \eq{eq:3NN}.  Suppose that the dimension of the output, $M$, is greater than 1 and $\thHz_*$ is a minimum of $L_{H_0}$. Let the matrices $G$, $F$ and the parameter $\thH_\vlambda$ be used in the same meaning as in Lemma \ref{lma:Hessian}. Then, if either of the conditions \\
(i) $G$ is positive or negative definite, and $F\neq O$,  \\
(ii) $G$ has positive and negative eigenvalues,  \\
holds, then for any $\vlambda$ with $\sum_{j=H_0}^H \lambda_j=1$ and $\lambda_j\neq 0$, $\thH_\vlambda$ is a saddle point of $L_H$.
\end{thm}

Theorem \ref{thm:flat} is easily proved from Lemma \ref{lma:Hessian}.  From the form of the lower-right four blocks of \eq{eq:Hessian}, it has positive and negative eigenvalues if $\tilde{G}$ is positive (or negative) definite and $\tilde{F}\neq O$.  See Section \ref{sec:proof_flat} in Supplements for a complete proof. 
The assumption $M\geq 2$ is necessary for the condition (i) to happen. In fact, \cite{localmin} discussed the case of $M=1$, in which $F=O$ is derived. The paper also gave a sufficient condition that the embedded point $\thH_\lambda$ is a local minimum when $G$ is positive (or negative) definite.  See Section \ref{sec:localmin_smooth_M1} for more details on the special case of $M=1$. 

Suppose that $\thHz_*$ attains zero training error.  Then, $\thH_\vlambda$ can never be a saddle point but a global minimum. Therefore,  the situation (ii) can never happen. In that case, if $G$ is invertible, it must be positive definite and $F=O$. We will discuss this case further in Section \ref{sec:zero_error}.


\section{Semi-flat minima by embedding of ReLU networks}
\label{sec:min_relu}

This section discusses networks with ReLU. Its special shape causes different results.  Let $\phi(t)$ be the ReLU function: $\phi(t) = \max\{t,0\}$, which is used very often in DNNs to prevent vanishing gradients \cite{NairHinton_ICML2010_ReLU,pmlr-v15-glorot11a}.  
The activation is given by $\varphi(\vx; \vw) = \phi(\vw^T \tilde{\vx})$ with $\vw^T \tilde{\vx}:=\vw_{wgt}^T \vx - w_{bias}$. 
It is important to note that the ReLU function satisfies positive homogeneity; i.e., $\phi(\alpha t) = \alpha\phi(t)$ for any $\alpha\geq 0$.
This causes special properties on $\varphi$, that is,  
(a) $\varphi(\vx;r\vw) = r \varphi(\vx; \vw)$ for any $r\geq  0$, (b) $
\frac{\partial \varphi(\vx;\vw)}{\partial \vw}\Bigl|_{\vw=r\vw_*} = \frac{\partial \varphi(\vx;\vw)}{\partial \vw}\Bigl|_{\vw=\vw_*}$ if $r> 0, \vw^T \tilde{\vx}\neq 0$,
and (c) $\frac{\partial^2 \varphi(\vx; \vw)}{\partial{\vw}\partial{\vw}}= 0$   if $\vw^T \tilde{\vx}\neq 0$. 

From the positive homogeneity, effective parameterization needs some normalization of $\vv_j$ or $\vw_j$.  However, this paper uses the redundant parameterization. In our theoretical arguments, no problem is caused by the redundancy, while it gives additional flat directions in the parameter space.

\subsection{Embeddings of ReLU networks}
\label{sec:inv_relu}
Reflecting the above special properties, we introduce modified versions for embeddings of $\thHz_*$.

{\bf (I)$_R$ Unit replication: }
Fix $\mathcal{U}_{H_0}^q$, and take $\vgamma=(\gamma_{H_0},\ldots,\gamma_H)\in \R^{H-H_0+1}$ and $\vbeta=(\beta_{H_0},\ldots,\beta_{H})$ such that $\beta_j>0$ ($H_0\leq\forall  j\leq H$) and $    \sum_{j=H_0}^{H}\gamma_j\beta_j =1$.
Define $\thH_{\vgamma,\bm{\beta}}$ by 
\begin{align}
    & \vw_i = \vu_i, \quad \vv_i =  \vzeta_i \quad (1\leq i\leq H_0-1),\nonumber \\
    & \vw_{j} = \beta_j \vu_{H_0}, \quad  \vv_j = \gamma_j \vzeta_{H_0}\quad (H_0\leq j\leq H). 
    \label{eq:th_gb}
\end{align}

{\bf (II)$_R$ Inactive units: }
Define a parameter $\hat{\vtheta}^{(H)}$ by 
\begin{align}\label{eq:inactive_unit_relu}
   & \vw_i  = \vu_{i},  \quad \vv_i  = \vzeta_{i} \quad  (1\leq i\leq H_0), \qquad 
       \vv_j:  \text{ arbitrary }\quad (H_0+1\leq j\leq H) \nonumber \\
   &  \vw_j  \;\text{ such that }\; \vw_j^T \tilde{\vx}_\nu < 0 \quad (\forall \nu, H_0+ 1\leq j\leq H).
\end{align}
The last condition is easily satisfied if $w_{bias}$ is large.  Note also that $\varphi(\vx_\nu;\vw_j)= 0$ for each $\nu$, but  $\varphi(\vx;\vw_j)\not\equiv 0$ in general. 
Since a small change of $\vw_j$ ($H_0+1\leq j\leq H$) does not alter $\varphi(\vx_\nu;\vw_j)=0$, the function $L_H$ is constant locally on $\vv_j$ and $\vw_j$ ($H_0+1\leq j\leq H$) at $\hat{\vtheta}^{(H)}$.  This is clear difference from the smooth case, where changing $\vw_j$ from $\vw^{(0)}$ may cause a different function. 

{\bf (III)$_R$ Inactive propagation: }
The inactive propagation is exactly the same as the smooth activation case.  The embedded point is denoted by $\tilde{\vtheta}^{(H)}$.

The following proposition is obvious from the definitions.  
\begin{prop}
For the unit replication and inactive propagation, we have $
\vf^{(H)}_{\thH_{\vgamma,\vbeta}} = \vf^{(H)}_{\tilde{\vtheta}^{(H)}} =  \vf^{(H_0)}_{\thHz_*}$. 
\end{prop}

We see that there are some other flat directions in addition to the general cases.  In the embedding by inactive units, if the condition $\vw_j^T\tilde{\vx}_\nu \leq 0$ is maintained, $L_H$ has the same value. Assume $\|\vx_\nu\|\leq 1$ without loss of generality, and fix $K>1$ as a constant. Define $\hat{\vw}_{j,wgt}={\bf 0}$ and $\hat{w}_{j,bias}=2K$ for $H_0+1\leq j\leq H$.  From  $\vw_j^T \tilde{\vx}_\nu \leq \|\vw_{j,wgt}\|-w_{j,bias}\leq 0$ for $\vw_j\in B_K:=\{\vw_j\mid \|\vw_{j,wgt}\|\leq K \text{ and }  K \leq w_{j,bias}\leq 3K\}$ and any $\vv_j$ ($H_0+1\leq j\leq H$), we have the following result, showing that an $(H-H_0)\times (M+D)$ dimensional affine subset at $\hat{\vtheta}^{(H)}$ gives the same value at $\vx_\nu$.
\begin{prop}\label{prop:inv_iu_relu}
Assume $\|\vx_\nu\|\leq 1$ ($\forall \nu$).  If $(\vv_i,\vw_i)=(\vzeta_{i*},\vu_{i*})$ ($1\leq i\leq H_0$) and $(\vv_j,\vw_j)\in \R^M\times B_K$  ($H_0+1\leq j\leq H$), we have for any $\nu=1,\ldots,n$
\[
\vf^{(H)}(\vx_\nu;\thH) = \vf^{(H_0)}(\vx_\nu;\thHz_*).
\]
\end{prop}

Next, for the unit replication of ReLU networks, 
the piecewise linearity of ReLU causes additional flat directions. To see this,  
for a fixed $(\vgamma,\vbeta)$ with $\sum_j \gamma_j\beta_j=1$, we introduce a parametrization in a similar manner to the smooth case.  
Let $A=(\alpha_{cj})$ be an $(H-H_0)\times (H-H_0+1)$ matrix such that
$\sum_{j=H_0}^H \alpha_{cj} \gamma_j\beta_j = 0$ ($\forall c$)
and $\bigl(\begin{smallmatrix} {\bf 1}_{H-H_0+1}^T  \\ A \end{smallmatrix} \bigr)$ is invertible. 
Fix such $A$ and define $(\va, \vxi_{H_0+1},\ldots,\vxi_H; \vb,\veta_{H_0+1},\ldots,\veta_{H})$ by \eq{eq:new_param}. 
The next proposition shows that a small change of $(\veta_j)_{j=H_0+1}^H$ does not alter the value $L_H(\thH)=L_{H_0}(\thHz_*)$.  
For $\delta>0$, let $B_\delta^\veta(\thH)$ denote the intersection of the ball of radius $\delta$ with center $\thH$ and the affine subspace spanned by $\veta_{H_0+1},\ldots,\veta_H$ at $\thH$.  See Section \ref{sec:proof_inv_relu2} in Supplements for the proof.  
\begin{prop}\label{prop:inv_relu2}
Let $\{\vx_\nu\}_{\nu=1}^n$ be any data set, $\thHz_*$ be any parameter of the ReLU network $\NN_{H_0}$, and $\thH_{\vgamma,\vbeta}$ be defined by \eq{eq:th_gb}.  Assume that $\vu_{H_0*}^T \vx_\nu\neq 0$ for all $\nu$.  Then, there is $\delta>0$ such that
$$
\vf^{(H)}(\vx_\nu;\thH) = \vf^{(H_0)}(\vx_\nu;\thHz_*)
$$
for any $\thH\in B_\delta^\veta(\thH_{\vgamma,\vbeta})$ and $\nu=1,\ldots,n$.
\end{prop}
The assumption $\vu_{H_0*}^T\vx_\nu\neq 0$ may easily happen in practice. (See Figure \ref{fig:gen_error}(a), for example.)


\subsection{Embedding a local minimum of ReLU networks}
\label{sec:localmin_relu}

We first consider the embedding of a minimum by inactive units. 
Let $\hat{\vtheta}^{(H)}$ be an embedding of $\thHz$ by \eq{eq:inactive_unit_relu}.  From Proposition \ref{prop:inv_iu_relu}, $L_H(\thH)$ does not depend on  $(\vv_j,\vw_j)_{j=H_0+1}^H$ around $\hat{\vtheta}^{(H)}$ but takes the same value as $L_{H_0}(\thHz)$ with $\thHz=(\vv_i,\vw_i)_{i=1}^{H_0}$.  We have thus the following theorem.  
\begin{thm}\label{thm:relu_flatmin}
Assume that $\thHz_*$ is a minimum of $L_{H_0}$. Then, the embedded point $\hat{\vtheta}^{(H)}$ defined by \eq{eq:inactive_unit_relu} is a minimum of $L_H$.  
\end{thm}

Theorem \ref{thm:relu_flatmin} and Proposition \ref{prop:inv_iu_relu} imply that there is an $(H-H_0)\times (M+D)$ dimensional affine subset that gives local minima, and in those directions $L_H$ is flat.  

Next, we consider the embedding by unit replication, which needs further restriction on $\vgamma$ and $\vbeta$.  Let $\vtheta^{(H_0)}$ be a parameter of $\NN_{H_0}$, and $\vgamma=(\gamma_j)_{j=H_0}^H$ satisfy $\sum_{j=H_0}^H \gamma_j > 0$.  Define $\thH_\vgamma$ by replacing $\vw_j=\beta_j \vu_j$ in \eq{eq:th_gb} with $\vw_j =  \vu_{H_0}/\sum_{k=H_0}^H\gamma_k$  ($H_0\leq j\leq H$). If we assume $\vu_{H_0*}^T\vx_\nu\neq 0$ ($\forall\nu$), the function $L_H$ is differentiable on $\veta_c, \vxi_c$, and 
%
for the same reason as Theorem \ref{thm:flat}, the derivatives are zero. By restricting the function on those directions around $\thH_\vgamma$,  from the fact $\frac{\partial^2 \varphi(\vx_\nu;\vu_{H_0})}{\partial\vu_{H_0}\partial\vu_{H_0}}=0$, we can see that the Hessian has the form $\bigl(\begin{smallmatrix} O & \tilde{F}\\ \tilde{F}^T & O\end{smallmatrix}\bigr)$, which includes a positive and negative eigenvalue unless $F=O$.  This derives the following theorem. (See Section \ref{sec:proof_saddle_relu} for a complete proof.) 
\begin{thm}\label{thm:relu_saddles}
Suppose that $\vtheta^{(H_0)}_*$ is a minimum point of $L_{H_0}$.  Assume that $\vu_{H_0*}^T \vx_\nu\neq 0$ for any $\nu=1,\ldots,n$, and that $F\neq O$ where $F$ is given by Lemma \ref{lma:Hessian}.  Then, for any $\vgamma\in\R^{H-H_0+1}$ such that $\sum_{j=H_0}^H \gamma_j > 0$, the embedded parameter $\thH_\vgamma$ is a saddle point of $L_H$.
\end{thm}

\section{Discussions}

\subsection{Minimum of zero error}
\label{sec:zero_error}

In using a very large network with more parameters than the data size, the training error may reach zero.  Assume $\ell(\vy,\vz)\geq 0$ and that a narrower model attains $L_{H_0}(\thHz_*)=0$ without redundant units, i.e., any deletion of a unit will increase the training error.  We investigate overparameterized realization of such a global minimum by embedding in a wider network $\NN_H$.  Note that by any methods the embedded parameter is a minimum.  This causes special local properties on the embedded point. 

For simplicity, we assume three-layer networks and $\|\vx_\nu\|\leq 1$ ($\forall \nu$).  First, consider the unit replication for the smooth activation.  As discussed in the last part of Section \ref{sec:embed_min}, the Hessian takes the form
\begin{equation}\label{eq:Hess_smooth_0}
\text{Smooth: }\quad  \nabla^2 L_H(\thH_\vlambda)=\kbordermatrix{
& \thHz & \veta_c & \vxi_c \\
& \nabla^2 L_{H_0}(\thHz_*)  &  O  & O \\ 
& O   &  O &   O \\ 
&  O   &  O  & \tilde{G}
},
\end{equation}
where $\tilde{G}$ is non-negative definite. It is not difficult to see (Section \ref{sec:Hess_iu}) that, in the case of inactive units, the lower-right four blocks take the form $\bigl( \begin{smallmatrix} O & O \\ O & S\end{smallmatrix}\bigr)$.  The case of inactive propagation is similar. 


For ReLU activation, assume $\thHz_*$ is a differentiable point of $L_{H_0}$ for simplicity.  From Proposition \ref{prop:inv_iu_relu}, the Hessian at the embedding $\hat{\vtheta}^{(H)}$ by inactive units is given by
\begin{equation}\label{eq:Hess_relu_0}
\text{ReLU: }\quad \nabla^2 L_H(\hat{\vtheta}^{(H)})=\kbordermatrix{
& \thHz & (\vv_j, \vw_j) \\
& \nabla^2 L_{H_0}(\thHz_*) &  O \\ 
& O  &    O 
}.
\end{equation}
By a similar argument to the smooth case, the Hessian for the unit replication $\thH_\gamma$ takes the same form as \eq{eq:Hess_relu_0}. 

\subsection{Generalization error bounds of embedded networks}
\label{sec:gen_error}
Based on the results in Section \ref{sec:zero_error}, here we compare the embedding between  ReLU and smooth activation.  The results suggest that the ReLU networks can have an advantage in generalization error when zero training error is realized by some type of overparameterized models. 

Suppose that the smooth model $\NN_{H_{0,s}}$ and ReLU mdoel $\NN_{H_{0,r}}$  attain zero training error without redundant units.  They are embedded by the method of inactive units into $\NN_{H_s}$ and $\NN_{H_r}$, respectively, so that $H_s-H_{0,s}=H_r-H_{0,r}(=:E)$ (the same number of surplus units).  The dimensionality of the parameters of $\NN_{H_{0,s}}$ and $\NN_{H_{0,r}}$ are denoted by $d_{sm}^0$ and $d_{rl}^0$, respectively. 

The major difference of the local properties in Eqs.~(\ref{eq:Hess_smooth_0}) and (\ref{eq:Hess_relu_0}) is the existence of matrix $S$ or $\tilde{G}$ in the smooth case.  The ReLU network has a flat error surface $L_H$ in both the directions of $\vw_j$ and $\vv_j$.  In this sense, the embedded minimum is {\em flatter} in the ReLU network. 
We relate this difference of semi-flatness  to the generalization ability of the networks through the PAC-Bayes bounds.  We give a summary here and defer the details in Section \ref{sec:pac-bayes}, Supplements.

Let $\mathcal{D}$ be a probability distribution of $(\vx,\vy)$ and $\mathcal{L}_H(\thH):=E_\mathcal{D}[\ell(\vy,\vf(\vx;\thH))]$ be the generalization error (or risk).  Training data $(\vx_1,\vy_1),\ldots,(\vx_n,\vy_n)$ are i.i.d.~sample with distribution $\mathcal{D}$.  Then, with a trained parameter $\hat{\vtheta}$, the PAC-Bayes bound tells 
\begin{equation}
     \mathcal{L}_H(\hat{\vtheta}) \lessapprox \frac{1}{n}L_H(\hat{\vtheta}) + 2\sqrt{\frac{2 ( KL(Q||P) + \ln\frac{2\delta}{n})}{n-1}},
\end{equation}
where $P$ is a prior distribution which does {\em not} depend on the training data, and $Q$ is any distribution such that it distributes on parameters that do not change the value of  $L_H$ so much from $L_H(\hat{\vtheta})$.  

We focus on the embedding by inactive units here. 
See Section \ref{sec:Hess_others}, Supplements, for the other cases. 
The essential factor of the PAC-Bayes bound is the KL-divergence $KL(Q||P)$, which is to be as small as possible. We use different choices of $P$ and $Q$ for the smooth and ReLU networks (see Section \ref{sec:pac-bayes} for details).  For the smooth networks, $P_{sm}$ is a non-informative normal distribution $N(0,\sigma^2 I_{d_{sm}})$ with $\sigma\gg 1$, and 
$Q_{sm}$ is $N(\hat{\vtheta}^{(H)}_{sm,0}, \tau^2 \mathcal{H}_{sm}^{-1})\times N(\hat{\vtheta}^{(H)}_{sm,1}, \sigma^2 I_{d^1})\times N(\hat{\vtheta}^{(H)}_{sm,2}, \tau^2 S^{-1})$ with $\tau\ll 1$, where the decomposition corresponds to the components $\thHz$,  $(\vv_j)_{j=H_0+1}^H$, and $(\vw_j)_{j=H_0+1}^H$.  $\mathcal{H}_{sm}:=\nabla^2 L_{H_0}({\thHz_{*,sm}})$ is the Hessian.  
For the ReLU networks, based on Proposition \ref{prop:inv_iu_relu}, $P_{rl}$ is given by $N(0,\sigma^2 I_{d^0_{rl}})\times N(0,\sigma^2 I_{d^1})\times \text{Unif}_{B_K^{E}}$, while $Q_{rl}$ is
$N(\hat{\vtheta}_{rl,0}^{(H)},\tau^2 \mathcal{H}^{-1}_{rl})\times N(\hat{\vtheta}_{rl,1}^{(H)},\sigma^2 I_{d^1})\times \text{Unif}_{B_K^{E}}$, 
where $d^{1}=E\times M$ is the dimensionality of $(\vv_j)_{j=H_0+1}^H$.  
For these choices, 
the major difference of the bounds is the term  
\[
    d^{1} \log \bigl(\sigma^2/\tau^2)
\]
in the KL divergence for the smooth model. We can argue that, in realizing perfect fitting to training data with an overparameterized network, the ReLU network achieves a better upper bound of generalization than the smooth network, when the numbers of surplus units are the same.  




{\bf Numerical experiments.} 
We made experiments on the generalization errors of networks with ReLU and $\tanh$ in overparameterization. The input and output dimension is $1$. Training data of size 10 are given by $\NN_1$ (one hidden unit) for the respective models with additive noise $\varepsilon\sim N(0,10^{-2})$ in the output.   We first trained three-layer networks with each activation to achieve zero training error ($<10^{-29}$ in squared errors) with minimum number of hidden units ($H_0=5$ in both models).  See Figure \ref{fig:gen_error}(a) for an example of fitting by the ReLU network.  We used the method of inactive units for embedding to $\NN_H$, and perturb the whole parameters with $N(0,\rho^2)$, where $\rho$ is the $0.01\times \|\thHz_*\|$.  The code is available in Supplements. Figure \ref{fig:gen_error}(b) shows the ratio of the generalization errors (average and standard error for 1000 trials) of $\NN_H$ over $\NN_{H_0}$ as increasing $H$.  We can see that, as more surplus units are added, the generalization errors increase for the $\tanh$ networks, while  the ReLU networks do not show such increase. This accords with the theoretical considerations in Section  \ref{sec:gen_error}. 

\begin{figure}
    \centering
    \includegraphics[width=4.8cm]{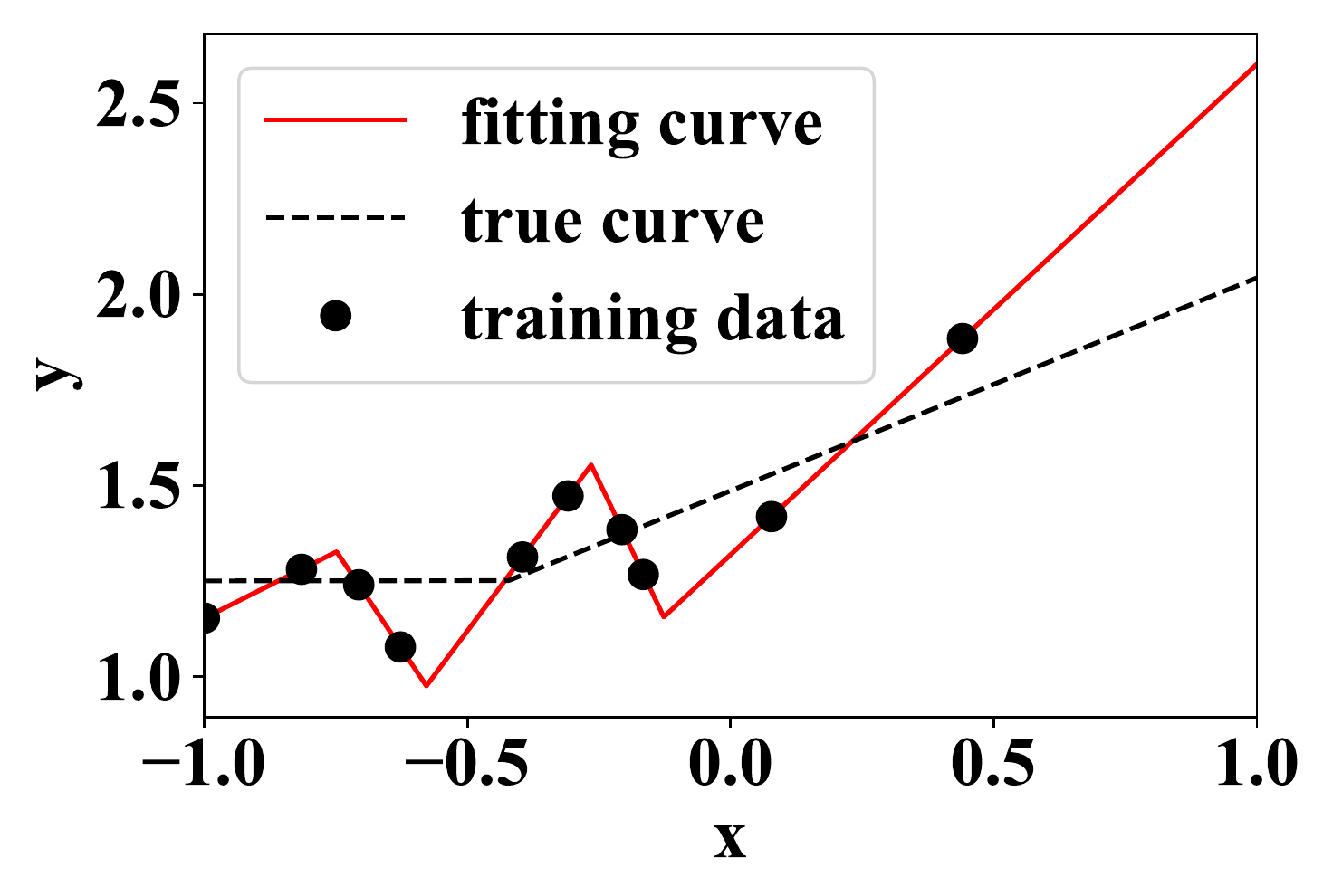}\hspace*{6mm}
    \includegraphics[width=4.8cm]{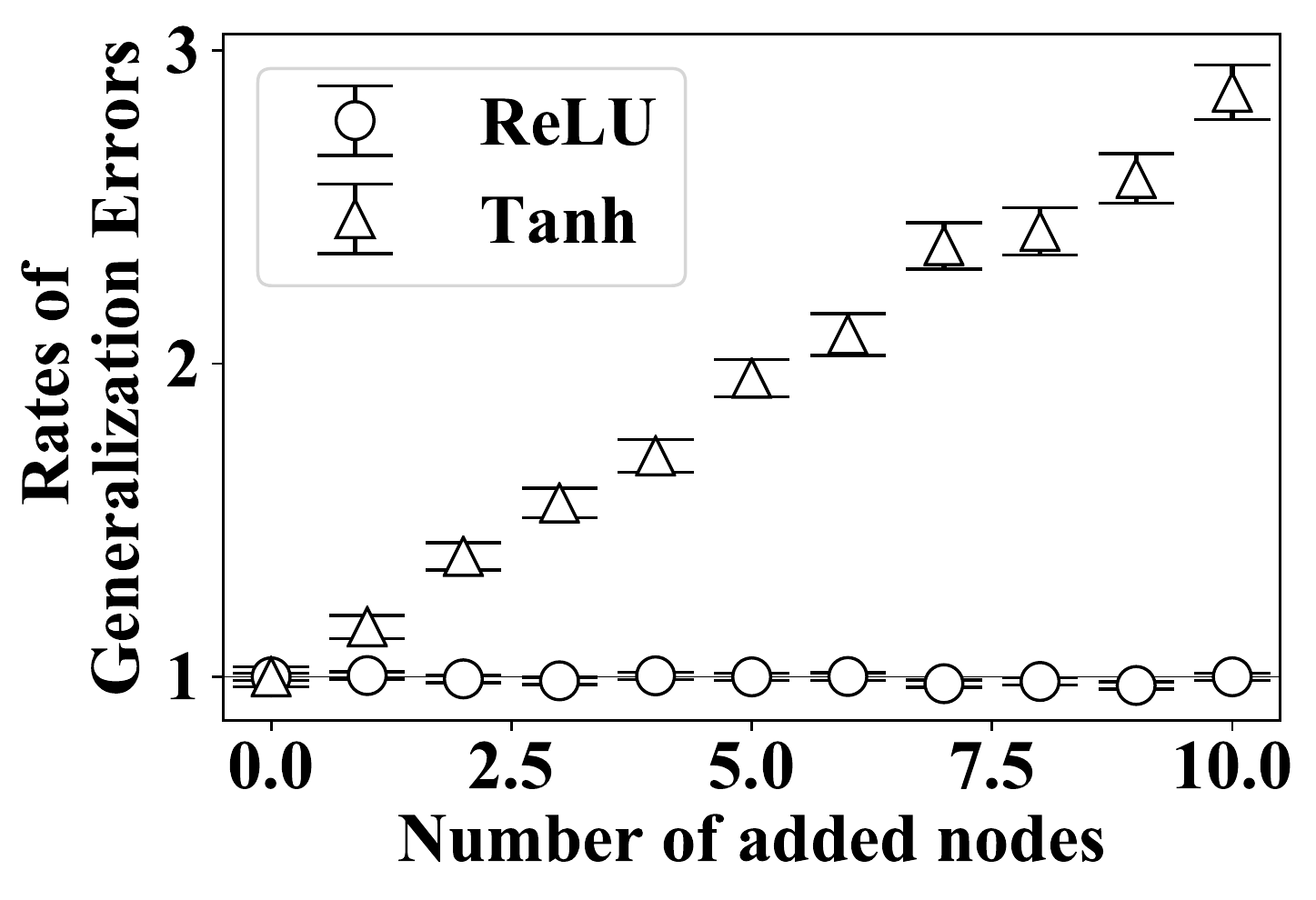}\\
    \vskip -2mm
    \hspace{3mm}(a) \hspace*{5cm} (b)
    \vspace*{-1mm}
    \caption{(a) Data and fitting by $\NN_5$ with ReLU. (b) Ratio of generalization errors of $\NN_H$ and $\NN_{H_0}$.}
    \vspace*{-2mm}
    \label{fig:gen_error}
\end{figure}

\subsection{Additional remarks}

{\bf Regularization.}  In training of a large network, one often regularizes parameters based on the norm such as $\ell_2$ or $\ell_1$.  Consider, for example, the inactive method of embedding for $\tanh$ or ReLU by setting $\vv_j = 0$ and $\vw_j=0$ ($H_0+1\leq j \leq H$).  Then the norm of the embedded parameter is smaller than that of unit replication.  This implies that if norm regularization is applied during training, the embedding by inactive units and propagation is to be promoted in overparameterized realization.  

{\bf Abundance of semi-flat minima in ReLU networks.} 
Theorems \ref{thm:relu_flatmin} and \ref{thm:relu_saddles} discuss three layer models for simplicity, but they can be easily extended to networks of any number of layers.  Given a minimum of $L_{H_0}$, it can be embedded to a wider network by making inactive units in any layers.
Thus, in a very large (deep and wide) network with overparameterization, there are many affine subsets of parameters to realize the same function, which consist of semi-flat minima of the training error.

\section{Conclusions}

For a better theoretical understanding of the error landscape, this paper has discussed three methods for embedding a network to a wider model, and studied overparameterized realization of a function and its local properties. From the difference of the properties between smooth and ReLU networks, our results suggest that ReLU may have an advantage in realizing zero errors with better generalization.  The current analysis reveals some nontrivial geometry of the error landscape, and its implications to dynamics of learning will be within important future works.

\bibliographystyle{abbrvnat}
\bibliography{fukumizu}


\clearpage

\include{supplements}

\end{document}

%% file: supplements.tex
\appendix

\begin{center}
{\LARGE Supplements to \\``Semi-flat minima and saddle points by embedding neural networks to overparameterization''}
\end{center}

\section{Proof of Theorem \ref{thm:stationary}}


We show a proof using the original parameterization.  We can also use the repameteriation introduced in Section \ref{sec:embed_min}, which may give other insights on the local properties, but we omit it here. See also Figure \ref{fig:DNN} for the meaning of parameters. 

Recall that the gradients of $L_H$ with respect to the parameters can be given by the {\em back-propagation}, which computes the derivatives with respect to the weight parameters successively from the output layer to the input. 
 For simplicity we use the notation
 \begin{equation}
     \ell_\nu(\thH):= \ell(\vy_\nu, \vf^{(H)}(\vx_\nu;\thH)).
 \end{equation}
 Let $\vz^{k,\nu}=(z_1^{k,\nu},\ldots,z_{H_k}^{k,\nu})^T$ be the input to the $H_k$ units in the $k$-th layer for $\vx_\nu$, i.e., 
\[
z^{k,\nu}_i = \sum_{j=1}^{H_k} w^k_{ij} \phi(z_j^{k-1,\nu}),
\]
where $w^k_{ij}$ is the weight parameter connecting from $\mathcal{U}_j^{k-1}$ to $\mathcal{U}_i^k$. 
Let 
\[
\delta^{k,\nu}_i := \frac{\partial \ell_\nu(\thH)}{\partial z^{k}_i}.
\]
Then, the back-propagation or generalized delta rule \cite{pdp} computes the derivatives by
\begin{equation}
    \delta^{k,\nu}_j  = \sum_{i=1}^{H_{k+1}} w^{k+1}_{ij}\delta^{k+1}_i\phi'(z^{k}_j), \qquad 
    \frac{\partial L_H(\thH)}{\partial w^k_{ij}}  =\sum_{\nu=1}^n \delta^{k,\nu}_{i}\phi(z^{k-1,\nu}_j).
    \label{eq:BP}
\end{equation}

Now consider the embedding using a unit in the $q$-th layer. 
Note that the output of any layer except $q$ in $\vf^{(H)}(\vx;\thH_\lambda)$ is equal to that of $\vf^{(H_0)}(\vx;\thHz_*)$, and the backpropagation of the both networks gives exactly the same $\delta^{k,\nu}_i$ to any $\mathcal{U}_{k,i}$ for $k>q$.  It follows that
\begin{equation}
   \frac{ \partial L_H(\thH)}{\partial V_0}\Bigl|_{\thH=\thH_\lambda} = \frac{ \partial L_{H_0}(\thHz)}{\partial V_0}\Bigl|_{\thHz=\thHz_*}=O.
\end{equation}

The derivatives of $L_{H_0}$ with respect to $\vzeta_j$ and $\vu_j$ ($1\leq j\leq H_0)$ are given by
\begin{align}
    \frac{\partial L_{H_0}(\thHz)}{\partial \vzeta_j} & = \sum_{\nu=1}^n \frac{\partial \ell_\nu(\thHz)}{\partial \vz^{q+1,\nu}} \frac{\partial \vz^{q+1,\nu}}{\partial \vzeta_j} 
     = \sum_{\nu=1}^n \vdelta^{q+1,\nu}\varphi(\vx_\nu; \vu_j,W_0) \label{eq:der1} \\
    \frac{\partial L_{H_0}(\thHz)}{\partial \vu_j} & = \sum_{\nu=1}^n \frac{\partial \ell_\nu(\thHz)}{\partial \vz^{q+1,\nu}}\frac{\partial \vz^{q+1,\nu}}{\partial \vu_j} 
     = \sum_{\nu=1}^n {\vdelta^{q+1,\nu}}^T\vzeta_j \frac{\partial\varphi(\vx_\nu; \vu_j, W_0)}{\partial \vu_j},  
\end{align}
where $\vdelta^{q+1,\nu}=(\delta^{q+1,\nu}_1,\ldots,\delta^{q+1,\nu}_{M})^T$. 

In the same manner, for $1\leq j\leq  H_0-1$, the derivatives of $L_H$ with respect to $\vu_j$ and $\vw_j$ are given by
\begin{align}
    \frac{\partial L_{H}(\thH)}{\partial \vv_j} & = \sum_{\nu=1}^n \frac{\partial \ell_\nu(\thH)}{\partial \vz^{q+1,\nu}} \frac{\partial \vz^{q+1,\nu}}{\partial \vv_j} 
     = \sum_{\nu=1}^n \vdelta^{q+1,\nu}\varphi(\vx_\nu; \vw_j,W_0)  \\
    \frac{\partial L_{H}(\thH)}{\partial \vw_j} & = \sum_{\nu=1}^n \frac{\partial \ell_\nu(\thH)}{\partial \vz^{q+1,\nu}}\frac{\partial \vz^{q+1,\nu}}{\partial \vw_j} 
     = \sum_{\nu=1}^n {\vdelta^{q+1,\nu}}^T\vv_j \frac{\partial\varphi(\vx_\nu; \vw_j, W_0)}{\partial \vw_j}.  \label{eq:der4}
\end{align}
It is obvious that these derivatives at $\thH=\thH_\lambda$ are equal to those of $L_{H_0}$ at $\thHz_*$, and thus equal to zero.  

For $H_0\leq j\leq H$, by the definition of $\thH_\lambda$, we have 
\begin{align}
      \frac{\partial L_{H}(\thH)}{\partial \vv_j}\Bigl|_{\thH_\lambda} & 
     = \sum_{\nu=1}^n \vdelta^{q+1,\nu}\varphi(\vx_\nu; \vw_j,W_0)\Bigl|_{\thH_\lambda} =  \sum_{\nu=1}^n \vdelta^{q+1,\nu}_*\varphi(\vx_\nu; \vu_{H_0*},W_{0*})  \\
    \frac{\partial L_{H}(\thH)}{\partial \vw_j}\Bigl|_{\thH_\lambda} & 
     = \sum_{\nu=1}^n {\vdelta^{q+1,\nu}}^T\vv_j \frac{\partial\varphi(\vx_\nu; \vw_j, W_0)}{\partial \vw_j}\Bigl|_{\thH_\lambda}
     = \lambda_j \sum_{\nu=1}^n {\vdelta^{q+1,\nu}_*}^T\vzeta_{H_0*} \frac{\partial\varphi(\vx_\nu; \vu_{H_0*}, W_{0*})}{\partial \vu_{H_0}},
\end{align}
which are zero from the stationary condition of $\thHz_*$.  We have also
\begin{align}
     &   \frac{\partial L_{H}(\thH)}{\partial W_0}\Bigl|_{\thH_\lambda}  
     = \sum_{\nu=1}^n \sum_{j=1}^{H}{\vdelta^{q+1,\nu}}^T\vv_j \frac{\partial\varphi(\vx_\nu; \vw_j, W_0)}{\partial W_0}\Bigl|_{\thH_\lambda}    \nonumber \\
     & =  \sum_{\nu=1}^n {\vdelta^{q+1,\nu}_*}^T \sum_{j=1}^{H_0-1} \vzeta_{j*} \frac{\partial\varphi(\vx_\nu; \vu_{j*}, W_{0*})}{\partial W_{0}} 
     + \sum_{\nu=1}^n {\vdelta^{q+1,\nu}_*}^T \sum_{j=H_0}^{H} \lambda_j \vzeta_{H_0*} \frac{\partial\varphi(\vx_\nu; \vu_{H_0*}, W_{0*})}{\partial W_{0}} \nonumber \\
     & = \sum_{\nu=1}^n {\vdelta^{q+1,\nu}_*}^T \sum_{j=1}^{H_0} \vzeta_{j*} \frac{\partial\varphi(\vx_\nu; \vu_{j*}, W_{0*})}{\partial W_{0}}  \\ \nonumber 
     & = \frac{\partial L_{H_0}(\thHz)}{\partial W_0}\Bigl|_{\thHz=\thHz_*} \\
     & = O,
\end{align}
which completes the proof.

\begin{figure}[t]
    \centering
    \includegraphics[keepaspectratio,height=6cm]{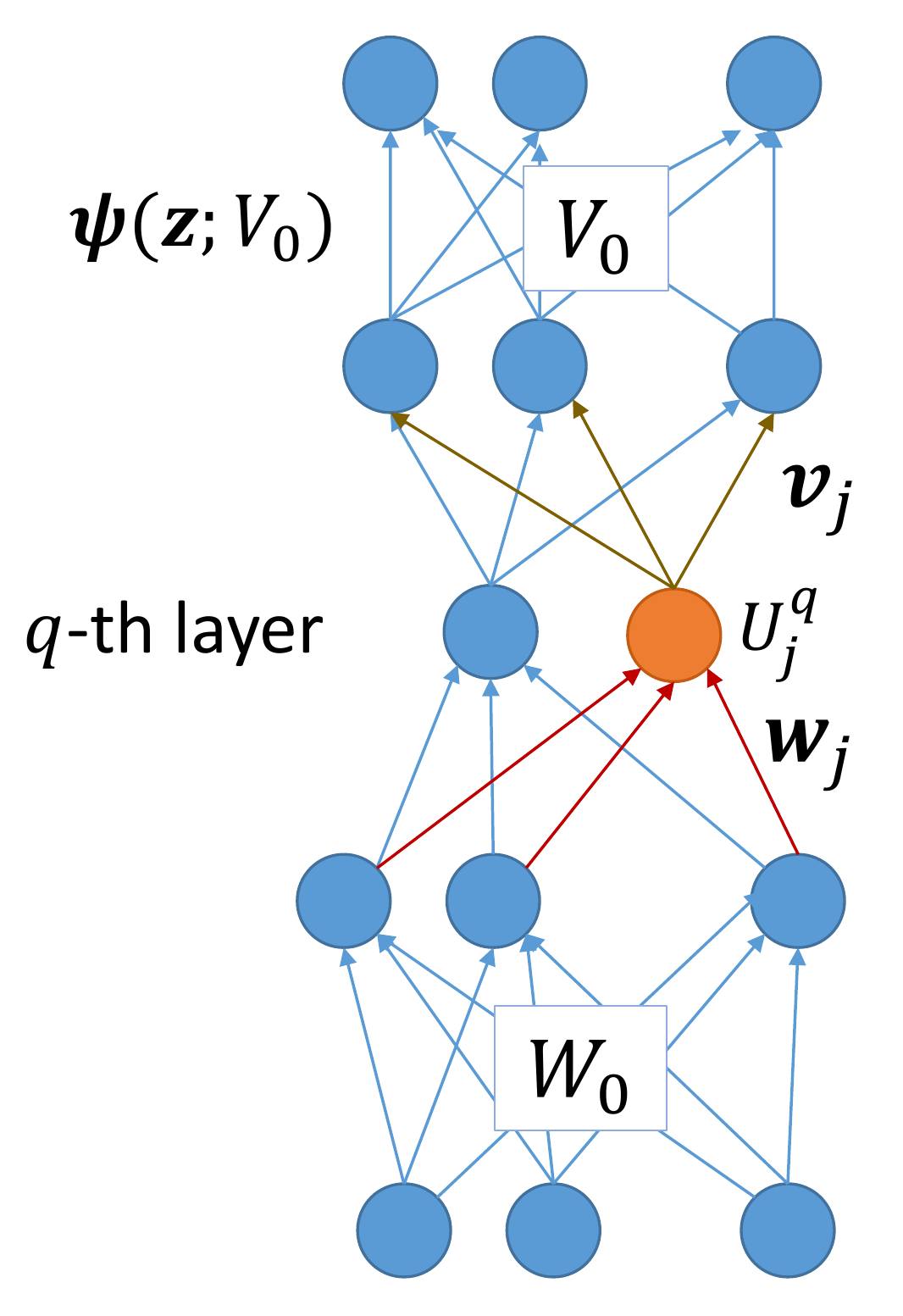}
    \caption{Function of neural networks}
    \label{fig:DNN}
\end{figure}

\section{Embedding by inactive units and propagation for smooth networks}
\label{sec:inactive_smooth}

As in Eqs.~(\ref{eq:der1}) through (\ref{eq:der4}), stationary conditions for $L_{H_0}$ give,  for $1\leq i\leq H_0$, 
\begin{align}
    \frac{\partial L_{H_0}(\thHz)}{\partial \vzeta_i} & 
     = \sum_{\nu=1}^n \vdelta^{q+1,\nu}\varphi(\vx_\nu; \vu_i,W_0) = {\bf 0} \nonumber \\
    \frac{\partial L_{H_0}(\thHz)}{\partial \vu_j} & 
     = \sum_{\nu=1}^n {\vdelta^{q+1,\nu}}^T\vzeta_i \frac{\partial\varphi(\vx_\nu; \vu_i, W_0)}{\partial \vu_i} ={\bf 0}.  \label{eq:der_H0}
\end{align}
The derivatives of $L_H$ with respect to $\vv_j$ and $\vw_j$ are given by
\begin{align}
    \frac{\partial L_{H}(\thH)}{\partial \vv_j} & = \sum_{\nu=1}^n \frac{\partial \ell_\nu(\thH)}{\partial \vz^{q+1,\nu}} \frac{\partial \vz^{q+1,\nu}}{\partial \vv_j} 
     = \sum_{\nu=1}^n \vdelta^{q+1,\nu}\varphi(\vx_\nu; \vw_j,W_0)  \label{eq:der_Hv} \\
    \frac{\partial L_{H}(\thH)}{\partial \vw_j} & = \sum_{\nu=1}^n \frac{\partial \ell_\nu(\thH)}{\partial \vz^{q+1,\nu}}\frac{\partial \vz^{q+1,\nu}}{\partial \vw_j} 
     = \sum_{\nu=1}^n {\vdelta^{q+1,\nu}}^T\vv_j \frac{\partial\varphi(\vx_\nu; \vw_j, W_0)}{\partial \vw_j}.  \label{eq:der_Hw}
\end{align}

In the case of inactive units, $\vv_j$ for $j\geq H_0+1$ is arbitrary and the $\frac{\partial\varphi(\vx_\nu; \vw^{(0)}, W_0)}{\partial \vw_j}$ is not necessarily zero, so that \eq{eq:der_H0} does not necessarily imply that \eq{eq:der_Hw} is zero.  In the case of inactive propagation, $vw_j$ is arbitrary for $j\geq H_0+1$, which does not mean \eq{eq:der_Hv} is zero in general. 

Consider the embedding by making both of units and propagation inactive; i.e., 
\begin{align}
    \vv_i & = \vzeta_i \quad (1\leq i\leq H_0) \nonumber \\
    \vw_i & = \vu_i \quad (1\leq i\leq H_0) \nonumber \\
    \vv_j & = {\bf 0} \quad (H_0+1\leq j\leq H) \nonumber \\
    \vw_j & = \vw^{(0)}   \quad (H_0+1\leq j\leq H). 
\end{align}
Then, for $j\geq H_0+1$, we have $\varphi(\vx;\vw_j,W_0)=0$  at $\vw_j=\vw^{(0)}$ which means \eq{eq:der_Hv} is zero, and \eq{eq:der_Hw} vanishes from $\vv_j=0$.  Therefore, the stationary point of $L_{H_0}$ is embedded to a stationary point of $L_H$, but there is no flat direction for this stationary point in general.

\section{Proofs of Lemmas \ref{lma:1st_der}, \ref{lma:Hessian}, and Theorem \ref{thm:flat} in Section \ref{sec:embed_smooth}}

In the sequel, we repeatedly use the following relations.
\begin{align}
&\frac{\partial }{\partial \vb}  =\sum_{j=H_0}^H \frac{\partial }{\partial \vw_j}, & 
\frac{\partial }{\partial \veta_c}  =\sum_{j=H_0}^H \alpha_{cj}\frac{\partial }{\partial \vw_j}, \nonumber \\ 
& \frac{\partial }{\partial \va}   = \sum_{j=H_0}^H \lambda_j \frac{\partial }{\partial \vv_j}, & 
\frac{\partial }{\partial \vxi_c}  = \sum_{k=H_0}^H \lambda_k \alpha_{ck}\frac{\partial }{\partial \vv_k}. \label{eq:der_repara}
\end{align}

\subsection{Proof of Lemma \ref{lma:1st_der}}
It follows from \eq{eq:der_repara} that 
\begin{align*}
    \frac{\partial \vf^{(H)}(\vx; \thH) }{\partial \vb}\Bigl|_{\thH=\thH_\lambda} & 
     =  \sum_{j=H_0}^H \frac{\partial \vf^{(H)}(\vx; \thH) }{\partial \vw_j}\Bigl|_{\thH=\thH_\lambda}  \nonumber \\
     & = \sum_{j=H_0}^H \vv_j  \frac{\partial \varphi(\vx;\vw_j) }{\partial \vw_j}\Bigl|_{\thH=\thH_\lambda} \nonumber \\
     & =  \sum_{j=H_0}^H \lambda_j \vzeta_{H_0}  \frac{\partial \varphi(\vx;\vu_{H_0}) }{\partial \vu_{H_0}} \nonumber \\
     & = \frac{\partial \vf^{(H_0)}(\vx; \thHo_*) }{\partial \vu_{H_0}},
\end{align*}
since $\sum_j \lambda_j=1$.  Also, 
\begin{align*}
    \frac{\partial \vf^{(H)}(\vx; \thH) }{\partial \veta_c}\Bigl|_{\thH=\thH_\lambda} & 
     =  \sum_{j=H_0}^H \alpha_{cj} \frac{\partial \vf^{(H)}(\vx; \thH) }{\partial \vw_j}\Bigl|_{\thH=\thH_\lambda}  \nonumber \\
     & = \sum_{j=H_0}^H \alpha_{cj}  \lambda_j \vzeta_{H_0}  \frac{\partial \varphi(\vx;\vu_{H_0}) }{\partial \vu_{H_0}} \; = 0,
\end{align*}
since $\sum_{j}\alpha_{cj}\lambda_j=0$ by definition of $A$. 

From \eq{eq:der_repara}, we have  
\begin{align*}
    \frac{\partial \vf^{(H)}(\vx; \thH) }{\partial \va}\Bigl|_{\thH=\thH_\lambda} & 
     =   \sum_{j={H_0}}^H \lambda_j \frac{\partial \vf^{(H)}(\vx; \thH) }{\partial \vv_j}\Bigl|_{\thH=\thH_\lambda}  \nonumber \\
     & = \sum_{j=H_0}^H \lambda_j \varphi(\vx;\vw_j) I \Bigl|_{\thH=\thH_\lambda} \nonumber \\
     & =   \varphi(\vx;\vu_{H_0,*}) I \nonumber \\
     & = \frac{\partial \vf^{(H_0)}(\vx; \thHz_*) }{\partial \vzeta_{H_0}},
\end{align*}
and
\begin{align*}
    \frac{\partial \vf^{(H)}(\vx; \thH) }{\partial \vxi_c}\Bigl|_{\thH=\thH_\lambda} & 
     =  \sum_{k=H_0}^H \lambda_k \alpha_{ck}\frac{\partial \vf^{(H)}(\vx; \thH) }{\partial \vv_k} \Bigl|_{\thH=\thH_\lambda}  \nonumber \\
     & = \sum_{k=H_0}^H \lambda_k\alpha_{ck} \varphi(\vx;\vu_{H_0}) I \;= 0.
\end{align*}

\subsection{Proof of Lemma \ref{lma:Hessian}}
We use the notation
\[
\vz_\nu = \vf^{(H)}(\vx_\nu; \thH).
\]

(i) First, we compute the blocks related to the derivative with respect to $\veta$. We have 
\begin{equation}\label{eq:d_eta}
  \frac{\partial L_H(\thH)}{\partial\veta_c}=\sum_{\nu=1}^n \frac{\partial \ell_\nu(\thH) }{\partial\vz_\nu}\frac{\partial \vz_\nu}{\partial \veta_c}  
     =  \sum_{\nu=1}^n \sum_{m=1}^M \frac{\partial \ell_\nu(\thH) }{\partial z_{\nu,m}} \sum_{j=H_0}^H \alpha_{cj} v_{jm} \frac{\partial \varphi(\vx_\nu;\vw_j)}{\partial \vw_j}.
\end{equation}
It follows from Eqs.~(\ref{eq:der_repara}) and (\ref{eq:d_eta}) that
\begin{align}
\frac{\partial^2 L_H(\thH)}{\partial{\veta_c}\partial\va} 
& = \sum_{k=H_0}^H \lambda_k \frac{\partial^2 L_H(\thH)}{\partial\veta_c\partial\vv_k} \nonumber  \\
& = \sum_{\nu=1}^n \sum_{m=1}^M\frac{\partial^2 \ell_\nu(\thH) }{\partial\vz_\nu\partial z_{\nu,m}} \sum_{k=H_0}^H \lambda_k \varphi(\vx_\nu; \vw_k) \sum_{j=H_0}^H \alpha_{cj} v_{jm} \frac{\partial \varphi(\vx_\nu;\vw_j)}{\partial \vw_j}  \nonumber  \\
   & \qquad + \sum_{\nu=1}^n \frac{\partial \ell_\nu(\thH) }{\partial\vz_\nu} \sum_{k=H_0}^H \alpha_{ck} \lambda_k  \frac{\partial \varphi(\vx_\nu;\vw_k)}{\partial \vw_k}. \label{eq:eta_a}
\end{align}
By inserting $\thH=\thH_\vlambda$, the first term is zero since $\vv_j=\lambda_j \vzeta_*$ and $\sum_j \alpha_{cj}\lambda_{j}=0$. The second term is also zero from $\sum_k \alpha_{ck}\lambda_k = 0$. 

Differentiation of \eq{eq:d_eta} with $\vb$ gives 
\begin{align}
   \frac{\partial^2 L_H(\thH)}{\partial\veta_c\partial\vb} & = 
   \sum_{\nu=1}^n \sum_{m,m'=1}^M \frac{\partial^2 \ell_\nu(\thH) }{\partial z_{\nu,m}\partial z_{\nu,m'}} \sum_{k=H_0}^H v_{km'} \frac{\partial \varphi(\vx_\nu;\vw_k)}{\partial \vw_k} \sum_{j=H_0}^H \alpha_{cj} v_{jm} \frac{\partial \varphi(\vx_\nu;\vw_j)}{\partial \vw_j}  \nonumber \\
   & \qquad +\delta_{jk} \sum_{\nu=1}^n \sum_{m=1}^M \frac{\partial \ell_\nu(\thH) }{\partial z_{\nu,m}}  \sum_{j=H_0}^H \alpha_{cj}v_{jm} \sum_{k=H_0}^H   \frac{\partial^2 \varphi(\vx_\nu;\vw_k)}{\partial \vw_k \partial\vw_k}.  \label{eq:eta_b}
\end{align}
At $\thH=\thH_\vlambda$, both the terms are zero for the same reason as \eq{eq:eta_a}.

Similarly, for $\vs_i=\vv_i$ or $\vw_i$ ($1\leq i\leq H_0-1$), 
\begin{align}
\frac{\partial^2 L_H(\thH)}{\partial{\veta_c}\partial\vs_i} 
& =  \frac{\partial^2 L_H(\thH)}{\partial\veta_c\partial\vs_i} \nonumber  \\
& = \sum_{\nu=1}^n \sum_{m=1}^M\frac{\partial^2 \ell_\nu(\thH) }{\partial\vz_\nu\partial z_{\nu,m}} \frac{\partial\vz_\nu}{\partial\vs_i} \sum_{j=H_0}^H \alpha_{cj} v_{jm} \frac{\partial \varphi(\vx_\nu;\vw_j)}{\partial \vw_j},  \nonumber  
\end{align}
which is zero at $\thH=\thH_\vlambda$ from $\sum_j \alpha_{cj}\lambda_j = 0$.

Next, from Eqs.~(\ref{eq:der_repara}) and (\ref{eq:d_eta}), we have 
\begin{align}
\frac{\partial^2 L_H(\thH)}{\partial\veta_c\partial\vxi_d} & 
= \sum_{\nu=1}^n \frac{\partial^2 \ell_\nu(\thH) }{\partial\vz_\nu\partial\vz_\nu} \sum_{k=H_0}^H \alpha_{dk}\lambda_k \varphi(\vx_\nu; \vw_k) \sum_{j=H_0}^H \alpha_{cj} \vv_j \frac{\partial \varphi(\vx_\nu;\vw_j)}{\partial \vw_j}  \nonumber  \\
   & \qquad + \sum_{\nu=1}^n \frac{\partial \ell_\nu(\thH) }{\partial\vz_\nu} \sum_{k=H_0}^H \alpha_{dk}\alpha_{ck} \lambda_k  \frac{\partial \varphi(\vx_\nu;\vw_k)}{\partial \vw_k}.  \label{eq:d_etaeta}
\end{align}
At $\thH=\thH_\lambda$, the first trem vanishes and the second term reduces to 
\begin{align*}
\frac{\partial^2 L_H(\thH_\lambda)}{\partial\veta_c\partial\vxi_d} = (A\Lambda A^T)_{cd}
\sum_{\nu=1}^n \frac{\partial \ell_\nu(\thHz_*) }{\partial\vz_\nu} \frac{\partial \varphi(\vx_\nu;\vu_{H_0,*})}{\partial \vu_{H_0}},
\end{align*}
which is $(A\Lambda A^T)_{cd} F$. 

\smallskip 
The block $\frac{L_H(\thH_\lambda)}{\partial{\veta_c}\partial{\veta_d}}$ can be computed in a similar way to \eq{eq:d_etaeta}:
\begin{align}
   \frac{\partial^2 L_H(\thH)}{\partial\veta_c\partial\veta_d} & = 
   \sum_{\nu=1}^n \sum_{m,m'=1}^M \frac{\partial^2 \ell_\nu(\thH) }{\partial z_{\nu,m'}\partial z_{\nu,m}}\sum_{j=H_0}^H \alpha_{cj} v_{jm} \frac{\partial \varphi(\vx_\nu;\vw_j)}{\partial \vw_j} \sum_{k=H_0}^H \alpha_{dk}v_{km'} \frac{\partial \varphi(\vx_\nu;\vw_k)}{\partial \vw_k}  \nonumber \\
   & \qquad + \sum_{\nu=1}^n \sum_{m=1}^M\frac{\partial \ell_\nu(\thH) }{\partial\vz_{\nu,m}}  \sum_{j=H_0}^H \alpha_{cj}\alpha_{dj} v_{jm}   \frac{\partial^2 \varphi(\vx_\nu;\vw_j)}{\partial \vw_j \vw_j}.  \nonumber
\end{align}
By plugging $\thH=\thH_\lambda$, the first term is zero, and the second term is reduced to 
\begin{equation}\label{eq:eta_eta}
    \sum_{j=H_0}^H \lambda_j\alpha_{cj}\alpha_{dj} \sum_{\nu=1}^n \frac{\partial \ell_\nu(\thH) }{\partial\vz_\nu}  \vzeta_{H_0,*}   \frac{\partial^2 \varphi(\vx_\nu;\vu_{H_0,*})}{\partial \vu_{H_0} \partial\vu_{H_0}},
\end{equation}
which is $(A\Lambda A^T)_{cd} G$.

\medskip
\noindent
(ii) Second, we will compute the remaining second derivatives including $\vxi_c$.  From \eq{eq:der_repara}, the first derivative with respect to $\vxi_c$ is given by 
\begin{equation}\label{eq:d_xi}
      \frac{\partial L_H(\thH)}{\partial\vxi_c}
     =  \sum_{\nu=1}^n \frac{\partial \ell_\nu(\thH) }{\partial\vz_\nu} \sum_{j=H_0}^H \lambda_j \alpha_{cj}\varphi(\vx_\nu;\vw_j) .
\end{equation}
From this expression, 
\begin{align*}
           \frac{\partial^2 L_H(\thH)}{\partial\vxi_c\partial\vv_k}\Bigl|_{\thH=\thH_\lambda} 
          & =  \sum_{\nu=1}^n\frac{\partial^2 \ell_\nu(\thH_\lambda) }{\partial\vz_\nu\partial\vz_\nu} \sum_{j=H_0}^H \lambda_j \alpha_{cj}\bigl(\varphi(\vx_\nu;\vu_{H_0,*})\bigr)^2\\
          & = 0,
\end{align*}
which means $\frac{\partial^2 L_H(\thH_\lambda )}{\partial\vxi_c\partial\vxi_d}$ and $\frac{\partial^2 L_H(\thH_\lambda)}{\partial\vxi_c\partial\va}$ are zero. 

It follows from Eqs.~(\ref{eq:d_xi}) and (\ref{eq:der_repara}) that
\begin{align*}
          & \frac{\partial^2 L_H(\thH)}{\partial\vxi_c\partial\vb}\Bigl|_{\thH=\thH_\lambda} \\
          & =  \sum_{\nu=1}^n \sum_{m=1}^M \frac{\partial^2 \ell_\nu(\thH_\lambda) }{\partial\vz_\nu\partial z_{\nu,m}}\sum_{j=H_0}^H \lambda_j \alpha_{cj}\varphi(\vx_\nu;\vu_{H_0,*}) \sum_{k=H_0}^H 
          v_{km}\frac{\partial\varphi(\vx_\nu;\vu_{H_0,*})}{\partial\vu_{H_0}} \\
          & \quad + \sum_{\nu=1}^n \frac{\partial \ell_\nu(\thH) }{\partial\vz_\nu} \sum_{j=H_0}^H \lambda_j \alpha_{cj} \frac{\partial\varphi(\vx_\nu;\vu_{H_0,*})}{\partial\vu_{H_0}}, 
\end{align*}
which is zero from $\sum_j \alpha_{cj}\lambda_j = 0$. 

It is also easy to see that for $\vs_i = \vv_i$ or $\vw_i$ ($1\leq i\leq H_0-1$) 
\[
\frac{\partial^2 L_H(\thH)}{\partial\vxi_c\partial\vs_i}\Bigl|_{\thH=\thH_\lambda} ={\bf 0}.
\]

\medskip
\noindent
(III) We compute the upper-left four blocks. We have 
\begin{equation}
  \frac{\partial L_H(\thH)}{\partial\va}  =  \sum_{\nu=1}^n   \frac{\partial \ell_\nu(\thH)  }{\partial\vz_\nu} \sum_{j=H_0}^H \lambda_j \varphi(\vx_\nu; \vw_j), 
\end{equation}
from which 
\begin{align*}
          \frac{\partial^2 L_H(\thH)}{\partial\va\partial\va}\Bigl|_{\thH=\thH_\lambda} 
          & = \sum_{\nu=1}^n \frac{\partial^2 \ell_\nu(\thHz_*) }{\partial\vz_\nu\partial\vz_\nu} \varphi(\vx_\nu;\vu_{H_0,*})^2 
           = \frac{\partial^2 L_{H_0}(\thHz_*)}{\partial\vzeta_{H_0}\partial\vzeta_{H_0}}
\end{align*}
and 
\begin{align*}
         &  \frac{\partial^2 L_H(\thH)}{\partial\va\partial\vb}\Bigl|_{\thH=\thH_\lambda} \\
         &  =  \sum_{\nu=1}^n \sum_{m=1}^M \frac{\partial^2 \ell_\nu(\thH) }{\partial\vz_\nu\partial z_{\nu,m}} \sum_{j=H_0}^H \lambda_j\varphi(\vx_\nu;\vw_j) \sum_{k=H_0}^H v_{km} \frac{\partial\varphi(\vx_\nu;\vw_k)}{\partial\vw_k}\Bigl|_{\thH=\thH_\lambda} \\
          & \qquad + \sum_{\nu=1}^n   \frac{\partial \ell_\nu(\thH)  }{\partial\vz_\nu} \sum_{j=H_0}^H \lambda_j \frac{\partial \varphi(\vx_\nu; \vw_j)}{\partial \vw_j}\Bigl|_{\thH=\thH_\lambda} \\
          &  =  \sum_{\nu=1}^n \frac{\partial^2 \ell_\nu(\thHz_*) }{\partial\vz_\nu\partial \vz_{\nu}} \vzeta_{H_0,*} \varphi(\vx_\nu;\vu_{H_0,*}) \frac{\partial\varphi(\vx_\nu;\vu_{H_0,*})}{\partial\vu_{H_0}}
          + \sum_{\nu=1}^n \frac{\partial \ell_\nu(\thHz_*)  }{\partial\vz_\nu}  \frac{\partial \varphi(\vx_\nu; \vu_{H_0,*})}{\partial \vu_{H_0}} \\
          & = \frac{\partial^2 L_{H_0}(\thHz_*)}{\partial\vzeta_{H_0}\partial\vu_{H_0}}.
\end{align*}

\smallskip 
\noindent
Finally, using 
$$
  \frac{\partial L_H(\thH)}{\partial\vb}=\sum_{\nu=1}^n \frac{\partial \ell_\nu(\thH) }{\partial\vz_\nu}\frac{\partial \vz_\nu}{\partial \vb}  
     =  \sum_{\nu=1}^n\sum_{m=1}^M \frac{\partial \ell_\nu(\thH) }{\partial z_{\nu,m}} \sum_{j=H_0}^H v_{jm} \frac{\partial \varphi(\vx_\nu;\vw_j)}{\partial \vw_j},
$$
we have 
\begin{align*}
 &  \frac{\partial^2 L_H(\thH)}{\partial\vb\partial\vb}\Bigl|_{\thH=\thH_\lambda} \\  
     & =  \sum_{\nu=1}^n \sum_{m,m'=1}^M\frac{\partial^2 \ell_\nu(\thH) }{\partial z_{\nu,m}\partial z_{\nu,m'}} \sum_{j=H_0}^H v_{jm} \frac{\partial \varphi(\vx_\nu;\vw_j)}{\partial \vw_j}\sum_{k=1}^H  v_{km'} \frac{\partial \varphi(\vx_\nu;\vw_k)}{\partial \vw_k}\Bigl|_{\thH=\thH_\lambda}  \\
    & \qquad  + \sum_{\nu=1}^n\sum_{m=1}^M\frac{\partial \ell_\nu(\thH) }{\partial z_{\nu,m}} \sum_{j=H_0}^H v_{jm} \frac{\partial^2 \varphi(\vx_\nu;\vw_j)}{\partial \vw_j\partial\vw_j}\Bigl|_{\thH=\thH_\lambda} \\
  & =  \sum_{\nu=1}^n \Bigl( \vzeta_{H_0,*}^T \frac{\partial^2 \ell_\nu(\thHz_*) }{\partial \vz\partial\vz} \vzeta_{H_0,*} \Bigr) \frac{\partial \varphi(\vx_\nu;\vu_{H_0,*})}{\partial \vu_{H_0}} \frac{\partial \varphi(\vx_\nu;\vu_{H_0,*})}{\partial \vu_{H_0}}  \nonumber \\
  & \qquad   + \sum_{\nu=1}^n\frac{\partial \ell_\nu(\thH) }{\partial\vz_\nu}  \vzeta_{H_0,*} \frac{\partial^2 \varphi(\vx_\nu;\vu_{H_0,*})}{\partial \vu_{H_0}\vu_{H_0}} \nonumber \\
  & = \frac{\partial^2 L_{H_0}(\thHz_*)}{\partial\vu_{H_0}\partial\vu_{H_0}}.
\end{align*}

\medskip
\noindent
(iv) Finally, it is similarly proved that for $\vs_i = \vv_i$ or $\vw_i$ ($1\leq i\leq H_0-1$) 
\begin{align*}
\frac{\partial^2 L_H(\thH)}{\partial\va\partial\vs_i}\Bigl|_{\thH=\thH_\vlambda} & =
\frac{\partial^2 L_{H_0}(\thHz_*)}{\partial\vzeta_{H_0}\partial\vs_i}   \nonumber \\
 \frac{\partial^2 L_H(\thH)}{\partial\vb\partial\vs_i}\Bigl|_{\thH=\thH_\vlambda} & = 
 \frac{\partial^2 L_{H_0}(\thHz_*)}{\partial\vu_{H_0}\partial\vs_i}.
\end{align*}

This completes the proof. 

\subsection{Proof of Theorem \ref{thm:flat}}
\label{sec:proof_flat}

Let $\tilde{F}:=(A\Lambda A^T)\otimes F$ and $\tilde{G}:=(A\Lambda A^T)\otimes G$.  Since $\lambda_j\neq 0$ ($\forall j$) and $A$ is of full rank, $(A\Lambda A^T)$ is of full rank.  
(i) Under the assumption, $\tilde{G}$ is invertible.  Then, the lower-right four blocks of the Hessian has the expression 
\begin{equation}\label{eq:G}
\begin{pmatrix}I & -\tilde{F}\tilde{G}^{-1} \\ O & I \end{pmatrix}
\begin{pmatrix}O & \tilde{F} \\ \tilde{F}^T & \tilde{G} \end{pmatrix}
\begin{pmatrix}I & O \\ -\tilde{F}\tilde{G}^{-1} & I \end{pmatrix}
= 
\begin{pmatrix} -\tilde{F}^T \tilde{G}^{-1} \tilde{F} & O \\ O & \tilde{G}. \end{pmatrix}.
\end{equation}
If $G$ is positive definite, so is $\tilde{G}$, and thus $-\tilde{F}^T \tilde{G}^{-1}\tilde{F}$ has negative eigenvalues for $F\neq O$. The Hessian of $L_H$ at $\thH_\vlambda$ has both of positive and negative eigenvalues, which implies $\thH_\vlambda$ is a saddle point.  The case of negative definite $G$ is similar. 
(ii) If $G$ has positive and negative definite, so does $\tilde{G}$.  This means that the Hessian of $L_H$ at $\thH_\vlambda$ has positive and negative eigenvalues.

\section{Local minima for smooth networks of 1-dimensional output}\label{sec:localmin_smooth_M1}

 The special property of $M=1$ is caused by vanishing $\tilde{F}$ in the Hessian.  In fact, the stationarity condition $\frac{\partial L_{H_0}(\thHz_*)}{\partial \vu_{H_0}}=0$ implies 
\[
  \zeta_{H_0,*}\sum_{\nu=1}^n \frac{\partial \ell_\nu(\thHz_*) }{\partial z_\nu} \frac{\partial \varphi(\vx_\nu;\vu_{H_0*})}{\partial \vu_{H_0}} = 0.
\]
Note that $\zeta_{H_0}$ is a scalar, and if we assume $\zeta_{H_0*}\neq 0$, the above condition implies $F=0$.  Then the corresponding part of the Hessian takes the form  
\[
\begin{pmatrix}O & O \\ O & \tilde{G}\end{pmatrix},
\]
which does not have negative eigenvalues if $G$ is non-negetive definite.   The zero blocks of the Hessian correspond to the directions $\vxi_c$ ($c=H_0+1,\ldots,H$), which make an affine subspace of $\Pi_{repl}(\thHz_\vlambda)$ having the same value $L_H(\thH)=L_{H_0}(\thHz_*)$.  Therefore, only the Hessian in the directions $(\va,\vb,\veta_{H_0+1},\ldots,\veta_H)$ matters to determine if $\thH_\lambda$ is a  minimum or saddle point.  
Note also that for $M\geq 2$ the stationarity condition gives 
\[
  \sum_{\nu=1}^n \sum_{m=1}^M  \frac{\partial \ell_\nu(\thHz_*) }{\partial z_{\nu,m}}\zeta_{H_0,m*} \frac{\partial \varphi(\vx_\nu;\vu_{H_0,*})}{\partial \vu_{H_0}} = 0,
\]
which does not necessary mean $F=O$.  


The following theorem is a slight extension of \citet[Theorem 3]{localmin}, in which only the case $H=H_0+1$ is discussed. 
\begin{thm}\label{thm:localmin_smooth_M1}
Suppose that the dimension of the output is 1 and $\thHz_*$ is a minimum of $L_{H_0}$ with positive definite Hessian matrix. In the following, the matrix $G$ and the parameter $\thH_\lambda$ are used in the same meaning as in Lemma \ref{lma:Hessian}. 
\begin{enumerate}
\item[(1)] Assume that the matrix $G$ is positive definite.
\begin{enumerate}
    \item [(a)] $\thH_\vlambda$ with $\sum_{j=H_0}^H \lambda_j = 1$ and $\lambda_j>0$ ($\forall j$)  is a minimum of $L_H$.    
    \item [(b)] $\thH_\vlambda$ with $\sum_{j=H_0}^H \lambda_j = 1$ and $\lambda_j<0 \text{ for some }j$ is a saddle point of $L_H$.
\end{enumerate}
\item[(2)] Assume that the matrix $G$ is negative definite. 
\begin{enumerate}
    \item [(a)]  If $\sum_{j=H_0}^H \lambda_j = 1$ and there is only one $i_0$  such that $\lambda_{i_0}>0$ and $\lambda_j < 0$ ($\forall j\neq i_0$), $\thH_\vlambda$ is a minimum of $L_H$.    
    \item [(b)] If $\sum_{j=H_0}^H \lambda_j= 1$ and $\lambda_j>0$ for at least two indices, $\thH_\vlambda$  is a saddle point of $L_H$.
\end{enumerate}
\item[(3)]  If the matrix $G$ has both of positive and negative eigenvalues, $\thH_\vlambda$ is a saddle point for any $\vlambda$ with $\sum_{a=H_0}^{H} \lambda_a = 1$ and $\lambda_a\neq 0$ ($\forall a$). 
\end{enumerate}
\end{thm}
\begin{proof}
For notational simplicity, the proof is given only for $H_0=1$; $\zeta_1$ and $\vu_1$ are written by $\zeta$ and $\vu$, respectively.  Extension to a general $H_0$ is easy and we omit it.   In the proof, let $\tilde{A}^T:=({\bf 1}_{H} A^T)$, which is invertible by assumption.  Note also that $\zeta, v_j$ are scalar parameters in the case of $M=1$.

\smallskip
\noindent
(1-a).  We first show that if $G$ is positive definite, the lower-right block of the Hessian, $\frac{\partial^2 L_H(\thH_\lambda)}{\partial{\veta}\partial{\veta}}=(A\Lambda A^T)\otimes G$, is positive definite.  This can be proved if $A\Lambda A^T$ is positive definite, since the eigenvalues of the tensor product is given by the products of respective eigenvalues of $A\Lambda A^T$ and $G$. By the assumptions,  $A\Lambda A^T$ is non-negative definite.  Suppose $A\Lambda A^T\bm{s}=0$ for $\bm{s}\in \R^{H-1}\backslash \{0\}$. Then, $A^T \bm{s}=0$, and this implies $\tilde{A}^T \tilde{\bm s}=0$ for  $\tilde{\bm{s}}=(\bm{s}^T,0)^T\in\R^H$.  This is impossible by the invertible assumption of $\tilde{A}$. 

Now consider the Hessian $\nabla^2 L_H(\thH_\vlambda)$ in Lemma \ref{lma:Hessian}.  It is obvious that this Hessian is non-negative definite, but not positive definite, as the blocks corresponding to $(\xi_j)_{j=2}^H$ are zero.  
Let $\Pi_{\thHz_*}$ be the $(H-1)$ dimensional affine plane in the parameter space of $\NN_{H}$ such that 
$$
\Pi_{\thHz_*}:=\{ (a,\xi_2,\ldots,\xi_H;\vb,\veta_2,\ldots,\veta_{H})\mid a=\zeta_*,\vb=\vu_*,\veta_2=\cdots=\veta_{H}=0 \}.
$$ 
This plane includes $\thH_\lambda$, and is parallel to the subspace spanned by $\xi_j$ axes.  The function $L_H$ takes the same value as $L_1(\thHo_*)$ on the whole of $\Pi_{\thHz_*}$. Thus, $\thH_\lambda$ is a minimum of $L_H$ if the Hessian is positive definite along 
the directions compliment to $\Pi_{\thHz_*}$ (see Figure \ref{fig:localmin}).  From Lemma \ref{lma:Hessian}, 
the Hessian at $\thH_\lambda$ along the directions $(\va,\vb, \veta_j)$ is given by 
$$
\begin{pmatrix}
 \frac{\partial^2 L_1(\thHo_*)}{\partial{\thHz_*}\partial{\thHz_*}}  & O \\ 
O & ( A\Lambda A^T) \otimes G
\end{pmatrix},
$$
which is positive definite.  This completes the proof of (1-a). 

\begin{figure}
    \centering
    \includegraphics[height=6cm]{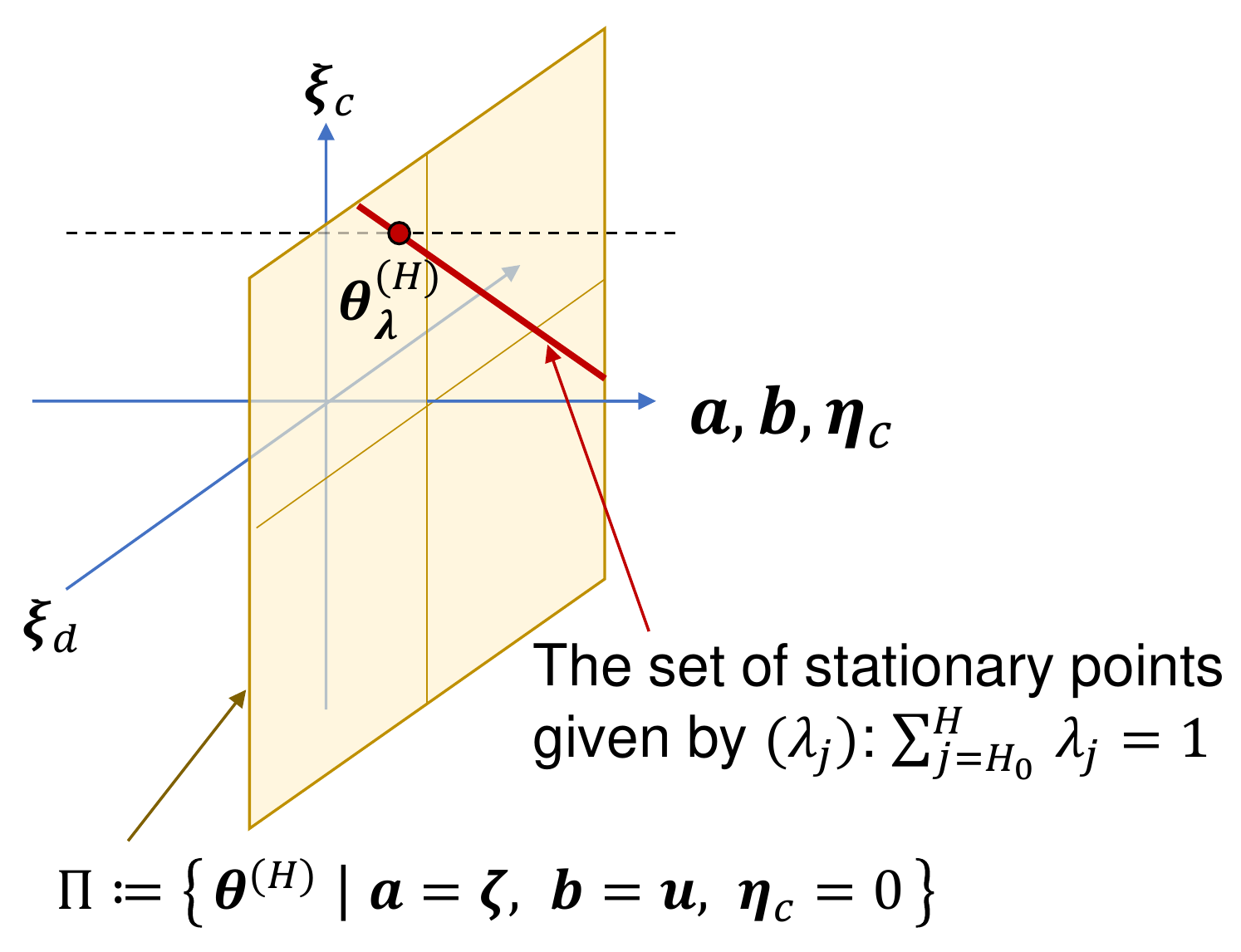}
    \caption{All the parameters on the affine subspace $\Pi$ has the same function as $\vf^{(H)}(\vx;\thH_\vlambda)$, and the affine subspace (in red) is a set of stationary points of $L_H(\thH)$.  The local behavior of $L_H$ around $\thH_\vlambda$ is determined by the second derivative along the $\va,\vb,\veta_c$ directions. }
    \label{fig:localmin}
\end{figure}

\smallskip
\noindent 
(1-b) From $A\lambda = 0$, it is easy to see that  
$$
\tilde{A}\Lambda\tilde{A}^T = \begin{pmatrix}1  & 0 \\ 0^T &  A\Lambda A^T \end{pmatrix}.
$$
Thus, the eigenvalues of $\tilde{A}\Lambda \tilde{A}^T$ is the eigenvalues of $A\Lambda A^T$ and $1$.  By Sylvester's law of inertia, the signature (the pair of the number of positive eigenvalues and that of negative ones) of $\tilde{A}\Lambda \tilde{A}$ coincides with the signature of $\Lambda$.  Since some $\lambda_i$ are negative by the assumption, $A\Lambda A^T$ has a negative eigenvalue.  Thus, under the assumption that $G$ is positive definite, $(A\Lambda A^T)\otimes G$ has a negative eigenvalue.  Since $\frac{\partial^2 L_H(\thH_\lambda)}{\partial\va\partial\va}$ is positive definite, the Hessian of $L_H(\thH)$ at $\thH=\thH_\lambda$ has positive and negative eigenvalues, which means $\thH_\vlambda$ is a saddle point.  

\smallskip
\noindent 
(2-a)
It suffices to show that $A\Lambda A^T$ is negative definite.  Then, $(A\Lambda A^T)\otimes G$ is positive definite, and the assertion is proved by the same argument as (1-a). Without loss of generality, we can assume that $\lambda_j<0$ for $1\leq j\leq H-1$ and $\lambda_H>0$.  Let $A=(A_0,\bm{h})$ where $A_0$ is an invertible matrix of size $H-1$, and let $\vlambda^T = (\vlambda_0^T, \lambda_H)$ with $\vlambda_0\in \R^{H-1}$.  The elements of $\vlambda_0$ are all negative by assumption. It follows that 
$$
A_0\vlambda_0 + \lambda_H \bm{h} = {\bf 0}, \qquad \sum_{j=1}^H \lambda_j = 1.
$$
A simple computation using $\bm{h}=-\frac{1}{\lambda_H}A_0 \vlambda_0$ provides
$$
A\Lambda A^T = A_0 \Bigl(\Lambda_0 + \frac{1}{\lambda_H}\vlambda_0 \vlambda_0^T \Bigr) A_0^T,
$$
where $\Lambda_0=\text{Diag}(\lambda_1,\ldots,\lambda_{H-1})$.  It is then sufficient to show that $B_0:= \Lambda_0 + \frac{1}{\lambda_H}\vlambda_0 \vlambda_0^T$ is negative definite.  If $\vs\in \R^{H-1}\backslash\{0\}$ is orthogonal to $\vlambda_0$, we have $\vs^T B_0 \vs = \vs^T\Lambda_0\vs < 0$.
Additionally, 
\begin{align*}
    \vlambda_0^T B_0 \vlambda_0  & = \sum_{j=1}^{H-1} \lambda_j^3 + \frac{1}{\lambda_H}\Bigl( \sum_{j=1}^{H-1}\lambda_j^2\Bigr)^2  \\
    & =  \frac{1}{\lambda_H}\Bigl\{  \Bigl(1-\sum_{j=1}^{H-1} \lambda_j \Bigr) \Bigl(\sum_{j=1}^{H-1} \lambda_j^3\Bigr)  +\Bigl( \sum_{j=1}^{H-1}\lambda_j^2\Bigr)^2\Bigr\} \\
    & = \frac{1}{\lambda_H}\Bigl\{  \Bigl(\sum_{j=1}^{H-1} \lambda_j^3\Bigr) + \sum_{i\neq j} \lambda_i^2 \lambda_j^2 - \sum_{i\neq j} \lambda_i \lambda_j^3 \Bigr\}  \\
    & = \frac{1}{\lambda_H}\Bigl\{  \Bigl(\sum_{j=1}^{H-1} \lambda_j^3\Bigr) + \sum_{i\neq j} \lambda_i^2 \lambda_j^2 - \sum_{i\neq j} \lambda_i \lambda_j \frac{\lambda_i^2+\lambda_j^2}{2} \Bigr\} \\
    & = \frac{1}{\lambda_H}\Bigl\{  \Bigl(\sum_{j=1}^{H-1} \lambda_j^3\Bigr) - \sum_{i\neq j} \frac{1}{2}\lambda_i \lambda_j (\lambda_i-\lambda_j)^2 \Bigr\} ,
\end{align*}
which is negative as well.  This proves the assertion.

\smallskip
\noindent
(2-b)  If there are two positive eigenvalues, the corresponding eigenspaces of at least two dimensions must intersects with the $H-1$ dimensional subspace spanned by the row vectors of $A$. Thus, $A\Lambda A^T$ has at least one positive eigenvalue, which means $(A\Lambda A^T)\otimes G$ has negative eigenvalues. The remaining proof is similar to (1-b).

\noindent 
(3) $A\Lambda A^T$ is of full rank, and thus $(A\Lambda A^T)\otimes G$ has both of positive and negative eigenvalues. The assertion is proved by the same argument as the case (1-b). 
\end{proof}

\section{Proof of Proposition \ref{prop:inv_relu2} and Theorem \ref{thm:relu_saddles} in Section \ref{sec:min_relu}}

\subsection{Proof of Proposition \ref{prop:inv_relu2}}
\label{sec:proof_inv_relu2}
First, note that, from $\vu_{H_0*}^T \vx_\nu\neq 0 (\forall \nu)$, there is $\delta>0$ such that for each $\vx_\nu$ the sign of $(\vu_{H_0,*}+\sum_{c=H_0+1}^H \alpha_{cj}\veta_c)^T \vx_\nu$ equals to that of $\vu_{H_0,*}^T\vx_\nu$ for any $j=H_0,\ldots,H$ and $(\veta_c)_{c=H_0+1}^H$ such that $\|(\veta_{H_0+1},\cdots,\veta_H)\|\leq \delta$.

Fix $\vx_\nu$, and assume first $\vu_{H_0,*}^T \vx_\nu > 0$.  Then, $(\vu_{H_0,*}+\sum_{c=H_0+1}^H \alpha_{cj}\veta_c)^T \vx_\nu > 0$ holds for $(\veta_c)_c$ with $\|(\veta_c)_c\|\leq \delta$. 
With the notation 
\begin{equation}\label{eq:embed_F}
\mathcal{F}_{H_0}:=\sum_{i=1}^{H_0-1} \vv_i \varphi( \vx_\nu;\vw_i) = \sum_{i=1}^{H_0-1}\vzeta_{i,*}\varphi(\vx_\nu; \vu_{i,*}),
\end{equation}
for any $\thH\in B_\delta^\veta(\thH_{\vgamma,\vbeta})$, we have 
\begin{align*}
\vf^{(H)}(\vx_\nu;\thH) & = 
    \mathcal{F}_{H_0} + \sum_{j=H_0}^H \gamma_j \vzeta_{H_0,*}\, \varphi\Bigl( \beta_j \bigl(\vu_{H_0,*}+\sum_{c=H_0+1}^H\alpha_{cj}\veta_c\bigr)^T \vx_\nu\Bigr) \\
    & = \mathcal{F}_{H_0} + \sum_{j=H_0}^H \gamma_j  \vzeta_{H_0,*}\, \beta_j\Bigl(\vu_{H_0,*}+\sum_{c=H_0+1}^H\alpha_{cj}\veta_c\Bigr)^T \vx_\nu  \\
    & = \mathcal{F}_{H_0} + \sum_{j=H_0}^H \gamma_j \beta_j\vzeta_{H_0,*}  \vu_{H_0,*}^T\vx_\nu +\vzeta_{H_0,*} \sum_{c=H_0+1}^H\sum_{j=H_0}^H\alpha_{cj} \gamma_j \beta_j  \veta_c^T\vx_\nu \\
    & = \mathcal{F}_{H_0} + \vzeta_{H_0,*} \vu_{H_0,*}^T\vx_\nu \nonumber \\
    & = \vf^{(H_0)}(\vx_\nu;\thHz_*),
\end{align*}
where we used $\sum_j \gamma_j \beta_j =1$ and $\sum_j \alpha_{cj} \gamma_j\beta_j = 0$. 

Next, if  $\vu_{H_0,*}^T \vx_\nu < 0$, we have 
$$
\vf^{(H_0)}(\vx_\nu;\thHz_*) = \mathcal{F}_{H_0},
$$
and 
\begin{align*}
\vf^{(H)}(\vx_\nu;\thH) & = 
    \mathcal{F}_{H_0} + \sum_{j=H_0}^H \gamma_j \vzeta_{H_0,*}\, \varphi\Bigl( \beta_j \bigl(\vu_{H_0,*}+\sum_{c=H_0+1}^H\alpha_{cj}\veta_c\bigr)^T \vx_\nu\Bigr)
    \; = \mathcal{F}_{H_0},  
\end{align*}
which completes the proof.

\subsection{Proof of Theorem \ref{thm:relu_saddles}}
\label{sec:proof_saddle_relu}

We use the same reparameterization $(\vv_1,\ldots,\vv_{H_0-1},\va,\vw_1,\ldots,\vw_{H_0-1},\vb,\vxi_{H_0+1},\ldots,\vxi_{H},\veta_{H_0+1},\ldots,\veta_H)$ as in Section \ref{sec:inv_relu} with $A\vgamma=0$.  We focus on the behavior of $L_H$ for a change of $\vxi_c,\veta_c$ with the others fixed at the values of $\thH_\vgamma$.  Note that, by the assumption $\vu_{H_0}^T\vx_\nu\neq 0$ for any $\nu$, $L_H(\thH)$ is differrentiable at $\thH_\vgamma$ with respect to $\vxi_c,\veta_c$.  By the same manner as Lemma \ref{lma:1st_der}, we have 
\[
\frac{\partial L_H(\thH)}{\partial\veta_c}\Bigl|_{\thH=\thH_\vgamma} = O,\qquad \frac{\partial L_H(\thH)}{\partial\vxi_c}\Bigl|_{\thH=\thH_\vgamma} = O,  
\]
which means $L_H$ is stationary at $\thH_\vgamma$ as a function of  $\veta_c$ and $\vxi_c$.  

From Lemma \ref{lma:Hessian}, we have 
\[
\frac{\partial^2 L_H(\thH)}{\partial\vxi_c\partial\vxi_d}\Bigl|_{\thH=\thH_\gamma} = O
\]
and 
\[
\frac{\partial^2 L_H(\thH)}{\partial\vxi_c\partial\veta_d}\Bigl|_{\thH=\thH_\gamma} = (A\Lambda A^T)_{cd}\sum_{\nu:\vu_{H_0*}^T\vx_\nu>0} \frac{\partial \ell_\nu(\thH_\vgamma)}{\partial\vz_\nu} \vx_\nu^T.
\]
Using the fact $\frac{\partial^2 \varphi(\vx_\nu;\vu_{H_0*})}{\partial \vu_{H_0}\vu_{H_0}}=0$, we have
\[ \frac{\partial^2 L_H(\thH)}{\partial\veta_c\partial\veta_d}\Bigl|_{\thH=\thH_\gamma} = O.
\]
Therefore, the Hessian of $L_H$ at $\thH_\vgamma$ with respect to $\vxi_a,\veta_b$ is given by
\[
\begin{pmatrix}
O & \tilde{F} \\ \tilde{F}^T & O
\end{pmatrix}
\]
where $\tilde{F}=(A \Lambda A^T)\otimes F$.  Under the assumption that $F\neq O$, the eigenvalues of the above Hessian are $\{\delta_i, -\delta_i\}_{i=1}^r$, where $\{\delta_i\}_{i=1}^r$ is the singular values of $\tilde{F}$. This means there are increasing directions and decreasing directions of $L_H$ around $\thH_\vgamma$, and thus it is a saddle point.

\section{PAC-Bayesian bound of generalization}
\label{sec:pac-bayes}

\subsection{Brief summary of general PAC-Bayes bound}
The PAC-Bayesian framework \cite{Mcallester_2003_pac-bayesian,McAllester1999} has been developed for bounding generalization performance of learning models.  It has been recently applied also to analysis of generalization of neural networks \cite{Neyshabur_etal_ICLR2018}. The following form of the bound is taken from \cite{Mcallester_2003_pac-bayesian}.

Let $f(\vx;\vtheta)$ be a real-valued function of $\vx$ with parameter $\vtheta\in\Theta$.  We consider the case that the loss function $\ell(\vy;\vz)$ is bounded, and without loss of generality assume $\ell(\vy,\vz)\in[0,1]$. Training data $(\vx_1,\vy_1),\ldots,(\vx_n,\vy_n)$ is an i.i.d.~sample from a distribution $\mathcal{D}$ on $(\vx,\vy)$.  Given function $f(\vx;\vtheta)$, the training error (or empirical risk) is evaluated by 
\[
\hat{L}(\vtheta) = \frac{1}{n}\sum_{\nu=1}^n \ell(\vf(\vx_\nu,\vtheta),\vy_\nu)
\]
and the generalization error (or risk) is defined by
\[
L(\vtheta) = E_\mathcal{D}[\ell(\vf(\vx_\nu,\vtheta),\vy_\nu) ].
\]

In PAC-Bayes bound, we introduce a "prior" distribution $P$ on the parameter space with an assumption that $P$ does not depend on the training sample, and an arbitrary probability distribution $Q$ on $\Theta$.  The distribution $Q$ may depend on the training sample.  Then, for any $\delta>0$, the inequality 
\begin{equation}\label{eq:pac-bayes_org}
E_{Q}[L(\vtheta)] \leq E_Q[\hat{L}(\vtheta)] + 2\sqrt{\frac{2 ( KL(Q||P) + \ln\frac{n}{\delta})}{n-1}}
\end{equation}
holds for sufficiently large $n$ with probability greater than $1-\delta$. 

First, we can see that, if the distribution of $Q$ is concentrated on a parameter set that gives very close values to $L(\hat{\vtheta})$ or $\hat{L}(\hat{\vtheta})$ at a parameter $\hat{\vtheta}$ obtained by learning, then we have 
\[
E_Q[L(\vtheta)]\approx L(\hat{\vtheta}), \qquad E_Q[\hat{L}(\vtheta)]\approx \hat{L}(\hat{\vtheta}).
\]
In such cases, \eq{eq:pac-bayes_org} shows the behavior of generalization error by its upper bound involving the approximate training error and the complexity term, which is expressed by the KL-divergence. 

\subsection{Generalization error bounds of embedded networks}
\label{sec:gen_error_detail}

The difference of the semi-flatness between networks of the smooth and ReLU activation can be related to the different generalization abilities of these models trough the PAC-Bayes bound \eq{eq:pac-bayes_org}. 

\subsubsection{Choice in general cases}
\label{sec:PQ_sharp}
First we consider the general problem of choosing $P$ and $Q$ appropriately when the minimum of $\hat{L}(\vtheta)$ is sharp (non-flat) and can be approximated locally by a quadratic function around $\hat{\vtheta}$, which is a minimum of $\hat{L}(\thH)$. The prior $P$ should be non-informative, and thus if $\Theta=\R^d$, a normal distribution $N(0,\sigma^2 I_d)$ with a large $\sigma$ is a reasonable choice.  
To relate the PAC-Bayes bound \eq{eq:pac-bayes_org} to the generalization error at $\hat{\vtheta}$, the distribution $Q$ (posterior) should distribute on parameters that do not change the empirical risk values so much from the values given by $\hat{\vtheta}$.  Under the assumption that $\hat{L}(\vtheta)$ is well approximated by a quardatic function, We set $Q$ by a normal distribution $N(0,\tau^2 \mathcal{H}^{-1})$ where 
$\mathcal{H}$ is the Hessian 
\[
\mathcal{H}:=\nabla^2 \hat{L}(\hat{\vtheta})
\]
with a small value of $\tau$.  Using the variance-covariance matrices based on the inverse Hessian is confirmed as follows. 
Suppose we set $Q$ by $N(\hat{\vtheta},\Sigma)$ with a general $\Sigma$ such that $\Sigma \ll \sigma^2$. Then, the Taylor series approximation of $\hat{L}(\thH)$ gives 
\[
    E_Q[\hat{L}(\thH)] \approx \hat{L}(\hat{\vtheta}^{(H)}) + \frac{1}{2} \tr[\mathcal{H}\Sigma],
\]
and thus the right hand side of \eq{eq:pac-bayes_org} is approximated by 
\begin{equation}\label{eq:Sigma}
\hat{L}(\hat{\vtheta}^{(H)}) + \frac{1}{2} \tr[\mathcal{H}\Sigma]+ 2\sqrt{\frac{2 ( KL(Q||P) + \ln\frac{n}{\delta})}{n-1}}.
\end{equation}
It is well known that $KL(Q||P)$ with $P$ and $Q$ normal distributions is given by 
\[
KL(Q||P) = \frac{1}{2} \Bigl[ \log\frac{|\sigma^2 I_d|}{|\Sigma|} + \tr[\sigma^{-2}\Sigma] + \frac{\|\hat{\vtheta}\|^2}{\sigma^2} - d \Bigr]
\]
To minimize \eq{eq:Sigma} with respect to  $\Sigma$, the differentiation provides the stationary condition 
\[
\mathcal{H} + \lambda \bigl( - \Sigma^{-1} + \sigma^{-2}I_d \bigr) = O
\]
with some positive constant $\lambda$.  From the assumption $\sigma^2 \gg \Sigma$, by neglecting $\sigma^{-2} I_d$, an approximate solution is given by 
\[
\Sigma_{opt} \approx \tau^2 \mathcal{H}^{-1},
\]
where $\tau>0$ is a scalar.  Plugging this to \eq{eq:Sigma} provides 
\[
\hat{L}(\hat{\vtheta}^{(H)}) + \frac{\tau^2}{2} d + 2\sqrt{\frac{2 \bigl\{ d\log \frac{\sigma^2}{\tau^2} + \log\det \mathcal{H} + \frac{\tau^2}{\sigma^2}\tr\bigl[\mathcal{H}^{-1}\bigr] +\frac{\|\hat{\vtheta}\|^2}{\sigma^2} - d\bigr\} + 2\ln\frac{n}{\delta}}{n-1}}.
\]
The second term is linear to $\tau^2$, and the main factor in the third term is $(d\log \frac{\sigma^2}{\tau^2} )^{1/2} n^{-1/2}$ when $\sigma\gg 1$ and $\tau\ll 1$.


\subsubsection{The case of inactive units}
\label{sec:Hess_iu}
We now discuss the embedding of the smooth and ReLU networks by inactive units when the training error achieves zero error.  As discussed in Section \ref{sec:zero_error}, some of the parameters give flat-directions, which requires some modification of the arguments in Section \ref{sec:PQ_sharp}. 

As notations, $\thH_{sm}\in\R^{d_{sm}}$ and $\thH_{rl}\in\R^{d_{rl}}$ are used for the parameters of networks with smooth and ReLU activation, respectively, and they are decomposed as $\thH_{sm}=(\thH_{sm,0}, \thH_{sm,1},\thH_{sm,2})$ and $\thH_{rl}=(\thH_{rl,0}, \thH_{rl,1}, \thH_{rl,2})$, corresponding to the components of a copy of $\thHz$, $(\vv_j)_{j=H_0+1}^H$, and $(\vw_j)_{j=H_0+1}^H$. Note that the both models have the same number of surplus parameters, i.e. $\text{dim}(\thH_{sm,1})=\text{dim}(\thH_{sm,2})=:d_1$ and $\text{dim}(\thH_{rl,1})=\text{dim}(\thH_{rl,2})=:d_2$.   
Different choices of $P$ and $Q$ are employed in the smooth and ReLU networks: we use $P_{sm}, Q_{sm}$ for the smooth networks and $P_{rl},Q_{rl}$ for the ReLU case. 

For the smooth activation, as in Section \ref{sec:PQ_sharp}, a non-informative prior
\[
P_{sm}:\quad N(0,\sigma^2 I)
\]
is used with $\sigma\gg 1$.  For the distribution $Q_{sm}$, 
we reflect the Hessian at the embedding by inactive units.  By the definition, the directions of $(\vv_j)_{j=H_0+1}^H$ give flat surface to $L_H$. The Hessian with respect to $(\vv_j,\vw_j)_{j=H_0+1}^H$ is thus given in the form 
\[
\begin{pmatrix}
O & O \\
O & S
\end{pmatrix},
\]
where $S$ is an $(H-H_0)\times D$ dimensional symmetric matrix given by 
\begin{multline*}
S_{jk} = \sum_{\nu=1}^n \vv_j^T\frac{\partial^2 \ell_\nu(\hat{\vtheta})}{\partial\vz\partial\vz} \vv_k \frac{\partial\varphi(\vx_\nu;\vw^{(0)})}{\partial\vw_j}\frac{\partial\varphi(\vx_\nu;\vw^{(0)})}{\partial\vw_k}
+ \delta_{jk}\sum_{\nu=1}^n \frac{\partial \ell_\nu(\hat{\vtheta})}{\partial\vz} \vv_j \frac{\partial^2\varphi(\vx_\nu;\vw^{(0)})}{\partial\vw_j\partial\vw_k}.
\end{multline*}
For the flat directions of $(\vv_j)_{j=H_0+1}^H$, the same distribution as $P$ is optimal for the upper bound.  Reflecting this, we set 
\[
Q_{sm}:\quad N(\hat{\vtheta}^{(H)}_{sm,0}, \tau^2 \mathcal{H}_{sm}^{-1})\times N(\hat{\vtheta}^{(H)}_{sm,1}, \sigma^2 I_{d^1})\times N(\hat{\vtheta}^{(H)}_{sm,2}, \tau^2 S^{-1}),
\]
where $\hat{\vtheta}^{(H)}_{sm}$ is the embedded point and $\mathcal{H}_{sm}:=\nabla^2 L_{H_0}({\thHz_{*,sm}})$ is the Hessian of the narrower network. 

For the ReLU networks, we first fix $K>1$ as a constant.  Since in the direction of $(\vw_j)_{j=H_0+1}^H$ we can presume the existence of the bonded flat subset $B_K^{H-H_0}$, we define the prior $P_{rl}$ by 
\[
P_{rl}:\quad N(0,\sigma^2 I_{d^0})\times N(0,\sigma^2 I_{d^1})\times \text{Unif}_{B_K^{H-H_0}}.
\]
Reflecting the flat directions, the posterior $Q_{rl}$ is defined by 
\[
Q_{rl}:\quad N(\hat{\vtheta}^{(H)}_{rl,0}, \tau^2 \mathcal{H}_{rl}^{-1})\times N(\hat{\vtheta}^{(H)}_{rl,1}, \sigma^2 I_{d^1})\times \text{Unif}_{B_K^{H-H_0}},
\]
where $\mathcal{H}_{rl}:=\nabla^2 L_{H_0}({\thHz_{*,rl}})$ is the Hessian of the narrower network.

With these choices, the KL divergence of the smooth case is given by 
\begin{multline*}
KL(Q_{sm}||P_{sm}) = \frac{1}{2} \Bigl[ d^0_{sm}\log\frac{\sigma^2 }{\tau^2}  + d^{1}\log\frac{\sigma^2}{\tau^2} + \log\det \mathcal{H}_{sm} + \log \det S \\
+ \tr\left[\frac{\tau^2}{\sigma^{2}}\bigl(\mathcal{H}_{sm}^{-1} +S^{-1}\bigr)\right] + \frac{\|\hat{\vtheta}_{sm}\|^2}{\sigma^2} - d^0_{sm}+d^{1} \Bigr],    
\end{multline*}
while in the case of ReLU networks,
\[
KL(Q_{rl}||P_{rl}) = \frac{1}{2} \Bigl[ d^0_{rl}\log\frac{\sigma^2 }{\tau^2} +  \log\det \mathcal{H}_{rl}  +\tr\left[\frac{\tau^2}{\sigma^{2}}\mathcal{H}_{rl}^{-1}\right]+ \frac{\|\hat{\vtheta}_{rl}\|^2}{\sigma^2} - d^0_{rl} \Bigr].
\]
With $\sigma\gg 1$ and $\tau\ll 1$, the major difference between these divergences comes from the term 
\[
    d^{1} \log \frac{\sigma^2}{\tau^2}
\]
in the smooth networks.  This suggests the advantage of the ReLU network in the overparameterized realization of zero training error in terms of the PAC-Bayesian upper bound of generalization error. 

\subsubsection{The Hessian for the zero error cases}
\label{sec:Hess_others}

We summarize the Hessian matrix for the embedding of a global minimum that attains zero training error.  For simplicity, we write only the four blocks corresponding to the surplus units. 

{\bf Smooth activation}

(I) Unit replication: As discussed in Sections \ref{sec:embed_min} and \ref{sec:zero_error}, the the part of the Hessian is given by 
\begin{equation}
\begin{pmatrix}
O & O \\ O & \tilde{G}
\end{pmatrix}.
\end{equation}

(II) Inactive units:  The part of the Hessian is given by
\begin{equation}
\begin{pmatrix}
O & O \\ O & S_1
\end{pmatrix},
\end{equation}
where 
\begin{align*}
(S_1)_{jk}  = \frac{\partial^2 L_H(\hat{\vtheta})}{\partial\vw_j\partial\vw_k}  
&  = \sum_{\nu=1}^n \vv_j^T\frac{\partial^2  \ell_\nu(\hat{\vtheta})}{\partial\vz_\nu\partial\vz_\nu} \vv_k
\frac{\partial\varphi(\vx_\nu;\vw^{(0)})}{\partial\vw}\frac{\partial\varphi(\vx_\nu;\vw^{(0)})}{\partial\vw}^T \\
& \qquad + \delta_{jk}\sum_{\nu=1}^n  \frac{\partial  \ell_\nu(\hat{\vtheta})}{\partial\vz_\nu}\vv_j \frac{\partial^2\varphi(\vx_\nu;\vw^{(0)})}{\partial\vw\partial\vw}.
\end{align*}

(III) Inactive propagations: 
The part of the Hessian is given by
\begin{equation}
\begin{pmatrix}
S_2 & O \\ O & O
\end{pmatrix},
\end{equation}
where
\begin{align*}
(S_2)_{jk}  = \frac{\partial^2 L_H(\hat{\vtheta})}{\partial\vv_j\partial\vv_k}  
&  = \sum_{\nu=1}^n \frac{\partial^2  \ell_\nu(\hat{\vtheta})}{\partial\vz_\nu\partial\vz_\nu}
\varphi(\vx_\nu;\vw_j)\varphi(\vx_\nu;\vw_k).
\end{align*}

We see that in all of the three cases the part of the Hessian for the surplus parameters contains a non-zero block. 

{\bf ReLU}

(I)$_R$ Unit replication: As discussed in Sections \ref{sec:localmin_relu}, the the part of the Hessian is given by $\bigl(\begin{smallmatrix}O & \tilde{F} \\ \tilde{F}^T & O \end{smallmatrix}\bigr)$.  Since the embedded point must not be a saddle, we have $\tilde{F}=O$. As a result, the part of the Hessian is constant zero.

(II)$_R$ Inactive units:  As discussed in Section \ref{sec:zero_error}, the part of the Hessian is zero. 

(III)$_R$ Inactive propagations: 
In this case, the part of the Hessian is given by
\begin{equation}
\begin{pmatrix}
S_3 & O \\ O & O
\end{pmatrix},
\end{equation}
where
\begin{align*}
(S_2)_{jk}  = \frac{\partial^2 L_H(\hat{\vtheta})}{\partial\vv_j\partial\vv_k}  
&  = \sum_{\nu=1}^n \frac{\partial^2  \ell_\nu(\hat{\vtheta})}{\partial\vz_\nu\partial\vz_\nu}
\varphi(\vx_\nu;\vw_j)\varphi(\vx_\nu;\vw_k)
\end{align*}
which is not necessarily zero unless $\varphi(\vx_\nu;\vw_j)=0$ for all $\nu$. 

We can see that the embedding by inactive units and unit replication give zero matrix for the part of Hessian, while the inactive propagation does not necessarily has zero matrix.

\bibliographystyle{abbrvnat}
\bibliography{fukumizu}
